%% file: supp.tex
\documentclass[twoside]{article}

\usepackage[accepted]{aistats2020}

\input{macros}

\runningtitle{Fully Decentralized Joint Learning}

\begin{document}

\twocolumn[

\aistatstitle{Fully Decentralized Joint Learning of\\ Personalized Models and
Collaboration Graphs}

\aistatsauthor{Valentina Zantedeschi \And Aur\'elien Bellet \And Marc Tommasi}

\aistatsaddress{GE – Global Research, \\1 Research Circle, Niskayuna, NY 12309~\thanks{This work was carried out
while the author was affiliated with Univ Lyon, UJM-Saint-Etienne, CNRS,
Institut d Optique Graduate School, Laboratoire Hubert Curien UMR 5516,
France} \And Inria, France \And Universit\'e de Lille \& Inria, France } ]

\input{abstract}

\input{introduction}

\input{stateoftheart2}

\input{general_setting}
\input{method}
\input{analysis}
\input{discovery}
\input{experiments}
\input{conclusion}

\input{acks}

\bibliographystyle{apalike}
\bibliography{paper}

\newpage
\onecolumn
\appendix

\section*{SUPPLEMENTARY MATERIAL}

This supplementary material is organized as follows. Section~\ref{app:fw}
provides the convergence analysis for our decentralized Frank-Wolfe
boosting algorithm.
Section~\ref{app:cd} describes the convergence analysis of our decentralized
graph learning algorithm. In Section~\ref{sec:smooth}, we derive the strong
convexity and smoothness parameters of $g(w) = \lambda\|w\|^2 - 
\mathbf{1}^\top\log(d(w) + \delta)$ needed to apply Theorem~\ref{the:disco}.
Section~\ref{sec:communcation-memory} discusses the communication and memory
costs of our method, with an emphasis on the scalability.
Finally, Section~\ref{app:exp} gives more details on our experimental setting
and present additional results.

\input{proof-fw}

\input{proof-cd}

\input{supp_comm_mem}

\input{supp_exp}

\end{document}

%% file: macros.tex
\usepackage{natbib}

\usepackage[utf8]{inputenc} 
\usepackage{booktabs}       
\usepackage{amsfonts}       
\usepackage{nicefrac}       
\usepackage{microtype}      
\usepackage{subfigure}
\usepackage{multirow}

\usepackage{graphicx}
\graphicspath{{figures/}}

\usepackage{amsmath,amsthm}
\usepackage{amsfonts}
\usepackage{bm}
\usepackage{xspace}
\usepackage{color}

\newtheorem{theorem}{Theorem} 
\newtheorem{lemma}{Lemma}
\newtheorem{proposition}{Proposition}
\newtheorem{remark}{Remark}

\DeclareMathOperator*{\argmax}{\mathit{arg}\,max}
\DeclareMathOperator*{\argmin}{\mathit{arg}\,min}
\DeclareMathOperator*{\sign}{\mathit{sign}}
\DeclareMathOperator*{\diam}{diam}
\DeclareMathOperator*{\prox}{\mathit{prox}}
\newcommand{\mtime}[1]{(#1)}
\newcommand{\normone}[1] {\bigl\| #1 \bigr\|_1}
\newcommand{\normonenormal}[1] {\| #1 \|_1}

\newcommand{\norm}[1] {\left\| #1 \right\|}
\newcommand{\kt}[3]{#1_{#2}^{\mtime{#3}}}
\newcommand{\expect}[1]{\mathbb{E}_{k \sim \mathcal{U}(\! 1 \!, K \!)}[#1]}
\newcommand{\shortexpect}[1]{\mathbb{E}_k[#1]}

\newcommand{\algoname}{Dada\xspace}
\newcommand{\algonameF}{Dada-Oracle\xspace}
\newcommand{\algonameL}{Dada-Learned\xspace}
\newcommand{\persolin}{Perso-lin\xspace}
\newcommand{\persolinF}{Perso-lin-Oracle\xspace}
\newcommand{\persolinL}{Perso-lin-Learned\xspace}
\newcommand{\localboost}{Local-boost\xspace}
\newcommand{\globalboost}{Global-boost\xspace}
\newcommand{\locallin}{Local-lin\xspace}
\newcommand{\globallin}{Global-lin\xspace}
\newcommand{\adaweight}{\eta} 
\newcommand{\wrt}{w.r.t.\@\xspace}
\newcommand{\ie}{i.e.\@\xspace}

\newcommand{\gap}{\mathrm{gap}}
\newcommand{\unit}{\bm{e}}
\newcommand{\pad}[2]{{#1}_{[#2]}}
\newcommand{\first}[1]{\textbf{#1}}
\newcommand{\second}[1]{\emph{\textbf{#1}}}

\usepackage[draft]{hyperref}       
\hypersetup{hidelinks}

\AtBeginDocument{
  \label{CorrectFirstPageLabel}
  
}

\newcommand\thfootnote[1]{%
  \begingroup
  \renewcommand\thefootnote{}\footnote{#1}%
  \addtocounter{footnote}{-1}%
  \endgroup
}

%% file: abstract.tex

\begin{abstract}
  We consider the fully decentralized machine learning scenario where many
  users with personal datasets collaborate to learn models through local
  peer-to-peer exchanges, without a central coordinator. We propose to train
  personalized models that leverage a collaboration graph describing the
  relationships between user personal tasks, which we learn jointly with
  the models. Our fully decentralized optimization procedure alternates
  between training nonlinear models given the graph in a greedy boosting
  manner, and updating the collaboration graph (with controlled sparsity) given the
  models. Throughout the process, users exchange messages only with a small
  number of peers (their direct neighbors when updating the models, and a few random
  users when updating the graph), ensuring that the procedure naturally scales
  with the number of users.
  Overall, our approach is communication-efficient and avoids exchanging personal data. 
  We provide an extensive analysis of the convergence rate, memory and communication
  complexity of our approach, and demonstrate its
  benefits compared to competing techniques
  on synthetic and real datasets.
\end{abstract}

%% file: introduction.tex
\section{INTRODUCTION}




In the era of big data, the classical paradigm is to build huge data
centers to collect and process user data. This centralized access
to resources and datasets is convenient to train machine learning models, but also comes with important drawbacks. The service
provider needs to gather, store and analyze the data on a large central
server, which induces high infrastructure costs.  As the server represents
a single point of entry, it must also be secure enough to prevent attacks that
could put the entire user database in jeopardy. On the user end, disadvantages
include limited control over one's personal data as well as possible privacy
risks, which may come from the aforementioned attacks but also from
potentially loose data governance policies on the part of service providers. 
A more subtle risk is to be trapped in a ``single thought'' model which fades
individual users' specificities or leads to unfair predictions for some of the
users.\thfootnote{$^1$ This work was carried out while the author was affiliated with Univ Lyon, UJM-Saint-Etienne, CNRS, Institut d Optique Graduate School, Laboratoire Hubert Curien UMR 5516, France}

For these reasons and thanks to the advent of powerful personal devices,
we are currently witnessing a shift to a different paradigm where data is kept
on the users' devices, whose computational resources are leveraged to train
models in a collaborative manner.
The resulting data is not typically balanced nor independent and identically
distributed across machines, and additional constraints arise when many
parties are involved. In particular, the specificities of each user result in
an increase in model complexity and size, and information needs to be
exchanged across users to compensate for the lack of local data.
In this context, communication is usually a major bottleneck, so that solutions aiming at reaching an agreement between user models or requiring a central coordinator should be avoided.


In this work, we focus on \emph{fully decentralized learning}, which has
recently attracted a lot of interest
\citep{Duchi2012a,Wei2012a,colin2016gossip,Lian2017b,Jiang2017a,D2,Lian2018}. In this
setting, users exchange information through local peer-to-peer exchanges in a sparse
communication graph without relying on a central server that aggregates updates
or coordinates the protocol. Unlike federated learning which requires such
central coordination
\citep{mcmahan2016communication,konevcny2016federated,kairouz2019advances}, fully
decentralized learning naturally scales to large numbers of users
without single point of failure or communication bottlenecks \citep{Lian2017b}.


The present work stands out from existing approaches in fully decentralized
learning, which train a single global model that may not be adapted to all
users. Instead, our idea is to leverage the fact that in many large-scale
applications (e.g., predictive modeling in smartphones apps), each user
exhibits distinct behaviors/preferences but is sufficiently similar to 
\emph{some} other peers to benefit from sharing information with them. We thus
propose to jointly discover the relationships between the personal
tasks of users in the form of a sparse \emph{collaboration graph} and learn
personalized models that leverage this graph to achieve better
generalization performance.
For scalability reasons, the collaboration graph serves as an overlay to
restrict the communication to pairs of users whose tasks appear to be
sufficiently similar.
In such a framework, it is crucial that the graph is well-aligned
with the underlying similarity between the personal tasks to ensure
that the collaboration is fruitful and avoid convergence to poorly-adapted
models.

We formulate the problem as the optimization of a joint
objective over the models and the collaboration graph, in which collaboration
is achieved by introducing a trade-off between (i) having the personal model
of each user accurate on its local dataset, and (ii) making the models and the
collaboration graph smooth with respect to each other. 
We then design and analyze a fully decentralized algorithm to
solve our collaborative problem in an alternating procedure,
 in which we iterate between updating personalized models given the current graph and updating the graph (with controlled sparsity) given the current models.
We first propose an approach to learn
personalized nonlinear classifiers as combinations of a set of base predictors
inspired from $l_1$-Adaboost \citep{shen2010dual}. In the proposed
decentralized algorithm, users greedily update their personal models by
incorporating a single base predictor at a time and send the update only to
their direct neighbors in the graph. We establish the convergence rate of the
procedure and show that it requires very low communication costs (linear in the number of edges in the graph and
\emph{logarithmic} in the number of base classifiers to combine).
We then propose an approach to learn a sparse
collaboration graph. From the decentralized system perspective, users update
their neighborhood of similar peers by communicating only with small random
subsets of peers obtained through a peer sampling service \citep{peersampling}.
Our approach is flexible enough to accommodate various graph regularizers
allowing to easily control the sparsity of the learned graph, which is key to
the scalability of the model update step. For strongly convex regularizers, we
prove a fast convergence for our algorithm and show how the number of random
users requested from the peer
sampling service rules a trade-off between communication and convergence
speed.

To summarize,  we propose the first approach to train in a fully decentralized
way, i.e. without any central server, personalized and nonlinear models in a
collaborative way while also learning the collaboration graph. Our main contributions are as follows. (1) We formalize the
problem of learning with whom to collaborate, together with personalized models
for collaborative decentralized learning. (2) We propose and analyze a fully
decentralized algorithm to learn nonlinear personalized models with low
communication costs. (3) We derive a generic and scalable
approach to learn sparse collaboration graphs in the decentralized setting. 
(4) We show that our alternating optimization scheme leads to better
personalized models at lower communication costs than existing methods on
several datasets.

%% file: stateoftheart2.tex
\section{RELATED WORK}
\label{sec:related}





\textbf{Federated multi-task learning.} 
Our work can be seen as multi-task learning (MTL) where each user is
considered as a task. In MTL, multiple tasks are learned simultaneously with
the assumption that a structure captures task relationships. A popular
approach in MTL is to jointly optimize models for all tasks while
enforcing similar models for similar tasks~
\citep{Evgeniou2004a,Maurer2006a,Dhillon2011b}. Task relationships are often
considered as known a priori but recent work also tries to learn this
structure \citep[see][and references therein]{zhang2017survey}.
However, in classical MTL approaches data is collected on a central server
where the learning algorithm is performed (or it is iid over the machines of a
computing cluster).
Recently, distributed and federated learning approaches~
\citep{wangmtl1,wangmtl2,baytas,smith2017federated} have been proposed to
overcome these limitations. 
Each node holds data for one task (non iid data) but these approaches still
rely on a central server to aggregate updates. 
The federated learning approach of
\citep{smith2017federated} is closest to our work for it jointly learns
personalized (linear) models and pairwise similarities across tasks. However,
the similarities are updated in a centralized way by the server which must
regularly access all task models, creating a significant
communication and computation bottleneck when the number of tasks is large.
Furthermore, the task similarities do not form a valid weighted graph and are
typically not sparse. This makes their problem formulation poorly
suited to the fully decentralized setting, where sparsity is key to ensure
scalability.

\textbf{Decentralized learning.}
There has been a recent surge of
interest in fully decentralized machine
learning. In most existing work, the goal is to learn the same global
model for all users by minimizing the average of the local objectives 
\citep{Duchi2012a,Wei2012a,colin2016gossip,lafond2016d,Lian2017b,Jiang2017a,D2,Lian2018}. In this
case, there is no personalization: the graph merely encodes the communication
topology without any semantic meaning and only affects the convergence
speed.
Our work is more closely inspired by recent decentralized approaches that have
shown the benefits of collaboratively learning personalized models for each
user by leveraging a
similarity graph given as input to the algorithm
\citep{vanhaesebrouck2016decentralized,li,bellet2017fast,Almeida2018}. 
As in our approach, this is achieved through a graph regularization
term in the objective. A severe limitation to the applicability of these
methods is that a relevant graph must be known beforehand, which is an
unrealistic assumption in many practical scenarios. Crucially, our approach
lifts this limitation by allowing to learn the graph along with the models.
In fact, as we demonstrate in our experiments, our decentralized graph
learning procedure of Section~\ref{sec:discovery} can be readily combined with
the algorithms of 
\citep{vanhaesebrouck2016decentralized,li,bellet2017fast,Almeida2018} in our
alternating optimization procedure, thereby broadening their scope.
It is also worth mentioning that 
\citep{vanhaesebrouck2016decentralized,li,bellet2017fast,Almeida2018} are
restricted to linear models and have per-iteration communication complexity
linear in the data dimension. Our boosting-based approach (Section~
\ref{sec:method}) learns nonlinear models with logarithmic communication
cost, providing an interesting alternative for problems of high
dimension and/or with complex decision boundaries, as illustrated in our experiments.

%% file: general_setting.tex
\section{PROBLEM SETTING AND NOTATIONS}
\label{sec:setting}

In this section, we formally describe the problem of interest.
We consider a set of users (or agents) $[K]=\{1,\dots,K\}$, each with a
personal data distribution over some common feature space $\mathcal{X}$
and label space $\mathcal{Y}$ defining a \emph{personal supervised
learning task}. For example, the personal task of each user could be to
predict whether
he/she likes a given item based on features describing the item. 
Each user $k$ holds a local dataset $S_k$ of
$m_k$ labeled examples drawn from its personal data distribution over $
\mathcal{X}\times \mathcal{Y}$, and aims to learn a model parameterized by
$\alpha_k\in\mathbb{R}^n$ which generalizes well to
new data points drawn from its distribution. We assume that all users learn
models from the same hypothesis class,
and since they have datasets of different sizes we introduce a notion
of ``confidence'' $c_k\in\mathbb{R}^+$ for each user $k$ which should be
thought of as proportional to $m_k$ (in practice we simply set $ c_k = m_k/\max_l m_l$).
In a non-collaborative setting, each user $k$ would typically select
the model parameters that minimize some (potentially regularized) loss
function $\mathcal{L}_k(\alpha_k; S_k)$ over its local dataset $S_k$.
This leads to poor generalization performance when local data is scarce.
Instead, we propose to study a collaborative
learning setting in which users discover relationships between their
personal tasks which are leveraged to learn better personalized models. We aim
to solve this problem in a fully decentralized way without relying on a
central coordinator node.

\textbf{Decentralized collaborative learning.}
Following the standard practice in the fully decentralized literature \citep{Boyd2006a}, each
user regularly becomes active at the ticks of an independent local clock
which follows a Poisson distribution.
Equivalently, we consider a global clock (with counter $t$) which ticks
each time one of the local clock ticks, which is convenient for stating and
analyzing the algorithms.
We assume that each user can send messages
to any other user (like on the Internet) in a peer-to-peer manner. However,
in order to scale to a large number of users and to achieve fruitful
collaboration, we consider a semantic overlay on the
communication layer whose goal is to restrict the message exchanges to pairs
of users whose tasks are most similar. We call this overlay a 
\emph{collaboration graph}, which is modeled as an undirected weighted
graph $\mathcal{G}_w = ([K], w)$ in which nodes
correspond to users and edge weights $w_{k,l}\geq0$ should reflect the
similarity between the learning tasks of users $k$ and $l$, with $w_{k,l}=0$
indicating the absence of edge. A user
$k$ only sends messages to its direct neighbors $N_k = \{ l : w_{k,l}>0 \}$ in
$\mathcal{G}_w$, and potentially to a small random set of peers
obtained through a peer sampling service \citep[see][for a decentralized
version]{peersampling}. Importantly, we do not enforce the graph
to be connected: different connected components can be seen as modeling
clusters of unrelated users.
In our approach, the collaboration graph is not known beforehand and
iteratively evolves (controlling its sparsity) in a learning scheme that
alternates between learning the graph and learning the models.
This scheme is designed to solve a global, joint optimization
problem that we introduce below.

\textbf{Objective function.}
We propose to learn the personal classifiers $\alpha=
(\alpha_1,\dots,\alpha_K)\in(\mathbb{R}^n)^K$ and the collaboration graph $w\in\mathbb{R}^{K
(K-1)/2}$ to minimize the following joint optimization problem:
\vspace{-2mm}
\begin{multline}
  \label{eq:general-joint-obj}
  \textstyle\min_{\substack{\alpha \in \mathcal{M}\\w\in\mathcal{W}}} 
  ~J(\alpha,w) = \sum_{k=1}^{K} d_k(w)c_k\mathcal{L}_k(\alpha_k; S_k) \\+ 
  \frac{\mu_1}
  {2}\textstyle\sum_
  {k<l} w_{k,l}\|\alpha_k-\alpha_l\|^2 + \mu_2 g(w),
\end{multline}
where $\mathcal{M}=\mathcal{M}_1\times\dots\times \mathcal{M}_K$ and $
\mathcal{W} = \{w\in \mathbb{R}^{K(K-1)/2} : w\geq 0\}$ are the feasible
domains for the models and the graph, $d(w)=
(d_1(w),\dots,d_K(w))\in\mathbb{R}^K$ is the degree vector with $d_k(w) = \sum_{l=1}^K w_{k,l}$, and
$\mu_1,\mu_2 \geq 0$ are trade-off hyperparameters.

The joint objective function $J(\alpha,w)$ in~\eqref{eq:general-joint-obj} is
composed of three terms. The first one is a (weighted) sum
of loss functions, each involving only the
personal model and local dataset of a single user.
The second term involves both the models and
the graph: it enables collaboration by encouraging two users $k$ and $l$
to have a similar model for large edge weight $w_{k,l}$.
This principle, known as graph regularization, is well-established in the
multi-task learning literature \citep{Evgeniou2004a,Maurer2006a,Dhillon2011b}.
Importantly, the factor $d_k(w)c_k$ in front of the local loss $\mathcal{L}_k$
of each user $k$ implements a useful inductive bias: users with larger
datasets (large confidence) will tend to connect to other nodes as long as
their local loss remains small so that they can positively influence their
neighbors, while users with small datasets (low confidence) will tend to disregard their local loss and rely more on
information from other users.
Finally, the last term $g(w)$ introduces some regularization on the
graph weights $w$ used to avoid degenerate solutions (e.g., edgeless
graphs) and control structural properties such as sparsity (see Section~\ref{sec:discovery}
for concrete examples).
We stress the fact that the formulation \eqref{eq:general-joint-obj} allows
for very flexible notions of relationships between the users' tasks. For
instance, as $\mu_1\rightarrow+\infty$ the problem becomes equivalent to
learning a shared model for all users in the same connected component of
the graph, by minimizing the sum of the losses of users independently in each component.
On the other hand, setting $\mu_1=0$ corresponds to having each user $k$ learn
its classifier $\alpha_k$ based on its local dataset only (no
collaboration). Intermediate values of $\mu_1$ let each user learn its own
personal model but with the models of other (strongly connected) users acting
as a regularizer.

While Problem \eqref{eq:general-joint-obj} is not jointly convex in $\alpha$
and $w$ in general, it is typically bi-convex. Our approach thus solves it by
alternating decentralized optimization on the models $\alpha$ and
the graph weights $w$.\footnote{Alternating optimization converges
to a local optimum under mild technical conditions, see 
\citep{Tseng01,Tseng_Yun09,Razaviyayn}.}

\textbf{Outline.} In
Section~\ref{sec:method}, we propose a
decentralized algorithm to learn nonlinear models given the graph in a greedy
boosting manner with communication-efficient updates. In Section~\ref{sec:discovery}, we design a decentralized algorithm to learn a 
(sparse) collaboration graph given the models with flexible regularizers $g
(w)$. We discuss related work in Section~\ref{sec:related}, and present
some experiments in Section~\ref{sec:exp}.



%% file: method.tex

\section{DECENTRALIZED COLLABORATIVE BOOSTING OF PERSONALIZED MODELS}
\label{sec:method}

In this section, given some fixed graph weights $w\in\mathcal{W}$, we propose
a decentralized algorithm for learning personalized nonlinear classifiers
$\alpha=(\alpha_1,\dots,\alpha_K)\in\mathcal{M}$ 
in a boosting manner which is essential to ensure only logarithmic communication complexity in the
number of model parameters while optimizing expressive models.
For simplicity, we focus on binary classification with $\mathcal{Y}=
\{-1,1\}$. We propose that each user $k$ learns a personal classifier as a
weighted combination of a set of $n$ real-valued base predictors $H =\{h_j : 
\mathcal{X} \to \mathbb{R} \}_{j=1}^{n}$, \ie a mapping $x\mapsto \sign
(\sum_{j=1}^n[\alpha_k]_jh_j(x))$ parameterized by $\alpha_k\in\mathbb{R}^n$.
The base predictors can be for instance
weak classifiers (e.g., decision stumps) as in standard boosting, or
stronger predictors pre-trained on separate data (e.g.,
public, crowdsourced, or collected from users who opted in to share
personal data).
We denote by $A_k \in 
\mathbb{R}^{m_k \times n}$ the matrix whose $(i,j)$-th entry gives the margin
achieved by the $j$-th base classifier on the $i$-th training
sample of user $k$, so that for $i\in[m_k]$, $
[A_k\alpha_k]_i=y_i\sum_{j=1}^n[\alpha_k]_jh_j(x_i)$ gives the margin
achieved by the classifier $\alpha_k$ on the $i$-th data point $(x_i,y_i)$ in
$S_k$. Only user $k$ has access to $A_k$.

Adapting the formulation of $l_1$-Adaboost~
\citep{shen2010dual,wang2015functional} to our personalized setting, we
instantiate the local loss $\mathcal{L}_k(\alpha_k; S_k)$ and the
feasible domain $\mathcal{M}_k=\{\alpha_k\in\mathbb{R}^n :\normonenormal{\alpha_k}\leq \beta\}$ for each user $k$ as follows:
\vspace{-2mm}
\begin{equation}
\label{eq:adaboost}
\mathcal{L}_k(\alpha_k; S_k) = \log\big( \sum_{i=1}^{m_k} e^{-
[A_k\alpha_k]_i}
\big),
\vspace{-2mm}
\end{equation}
where $\beta\geq 0$ is a hyperparameter to favor sparse models by
controlling their $l_1$-norm.
Since the graph weights are fixed in this section, with a slight
abuse of notation we denote by $f(\alpha):=J(\alpha,w)$ the objective
function in \eqref{eq:general-joint-obj} instantiated with the loss function
\eqref{eq:adaboost}. Note that $f$ is convex and
continuously differentiable, and the domain $
\mathcal{M}=\mathcal{M}_1\times\dots\times \mathcal{M}_K$ is a compact and convex subset of $(\mathbb{R}^n)^K$.

\subsection{Decentralized Algorithm}

We propose a decentralized algorithm based on
Frank-Wolfe (FW)~\citep{frank1956algorithm,pmlr-v28-jaggi13}, also known as
conditional gradient descent. Our approach is inspired from a recent FW
algorithm to solve $l_1$-Adaboost in the centralized
and non-personalized setting~\citep{wang2015functional}.  For clarity
of presentation, we set aside the decentralized setting for a moment and
derive the FW update with respect to the model of a single user.

\textbf{Classical FW update.}
Let $t\geq 1$ and denote by $\nabla[f(\alpha^{\mtime{t-1}})]_k$ the
partial
derivative of $f$ with respect to the $k$-th block of coordinates 
corresponding to the model $\alpha_k^{\mtime{t-1}}$ of user $k$. For step
size
$\gamma\in[0,1]$,
a FW update for user $k$ takes the
form of a convex combination $\alpha_k^
{\mtime{t}} = (1-\gamma) \alpha_k^{\mtime{t-1}} + \gamma s^{\mtime{t}}_k$ with
\begin{align}
\label{eq:fwupdate}
  s_k^{\mtime{t}} &= \textstyle\argmin_{\normonenormal{s}\leq \beta}~ s^\top \nabla[f(\alpha^{\mtime{t-1}})]_k \nonumber \\ 
  &= \beta
\sign(- (\nabla[f(\alpha^{\mtime{t-1}})]_k)_{j_k^{\mtime{t}}}) \unit^{j_k^{\mtime{t}}},
\end{align}
where $j_k^{\mtime{t}} = \argmax_j [
| \nabla[f(\alpha^{\mtime{t-1}})]_k |]_j$ and $\unit^{j_k^{\mtime{t}}}$ is the
unit vector with 1 in the $j_k^{\mtime{t}}$-th entry 
\citep{clarkson2010,pmlr-v28-jaggi13}. In other words, FW updates a single
coordinate of the current
model $\alpha_k^{\mtime{t-1}}$ which corresponds to the maximum absolute value entry
of the partial gradient $\nabla[f(\alpha^{\mtime{t-1}})]_k$. In our case, we
have:
\begin{multline}
\nabla[f(\alpha^{\mtime{t-1}})]_k = - d_k(w) c_k \adaweight_k^\top A_k + \mu_1(d_k(w)\alpha_k^{\mtime{t-1}} \\
- \textstyle\sum_l w_{k,l}\alpha_l^{\mtime{t-1}}),
\end{multline}
with $\adaweight_k = \frac{\exp(-A_k \alpha_{k}^{\mtime{t-1}})}{ \sum_{i=1}^{m_k} \exp(-A_k \alpha_{k}^{\mtime{t-1}})_i}$.
The first term in $\nabla[f(\alpha^{\mtime{t-1}})]_k$ plays the same role as
in standard Adaboost: the $j$-th entry (corresponding to the base predictor
$h_j$) is larger when $h_j$ achieves a large margin on the training sample
$S_k$ reweighted by $\adaweight_k$ (i.e., points that are currently poorly
classified get more importance). On the other hand, the more $h_j$ is used by
the neighbors of $k$, the larger the $j$-th entry of the second term. The FW
update~\eqref{eq:fwupdate} thus preserves the flavor of boosting 
(incorporating a single base classifier at a time which performs well on the
reweighted sample) with an additional bias towards selecting base predictors
that are popular amongst neighbors in the collaboration graph. The relative
importance of the two terms depends on the user confidence $c_k$.

\textbf{Decentralized FW.}
We are now ready to state our decentralized FW algorithm to optimize $f$.
Each user corresponds keeps its personal dataset locally. The fixed
collaboration graph $\mathcal{G}_w$
plays the role of an overlay: user $k$ only needs to communicate with its
direct neighborhood $N_k$ in $\mathcal{G}_w$. The size of $N_k$, $|N_k|$, is
typically small so that updates can occur in parallel in different parts of
the network, ensuring that the procedure scales well with the number of users.

Our algorithm proceeds as follows. Let us denote by
$\alpha^{\mtime{t}}\in\mathcal{M}$ the current models at time step $t$. Each
personal classifier is initialized to some feasible
point $\alpha_k^{\mtime{0}}\in\mathcal{M}_k$ (such as the zero vector). Then, at each step
$t\geq 1$, a
random user $k$ wakes up and performs the following actions:
\begin{enumerate}
\item \emph{Update step:} user $k$ performs a FW update on its local
model based on the most recent information $\alpha_l^{\mtime{t-1}}$ received from
its neighbors $l\in N_k$:
\begin{gather*}
\alpha_k^{\mtime{t}} = (1 - \kt{\gamma}{}{t}) \kt{\alpha}{k}{t-1} + 
\kt{\gamma}{}{t} \: \kt{s}{k}{t},\\ \text{with }\kt{s}{k}{t}\text{ as in
}\eqref{eq:fwupdate}\text{ and }\kt{\gamma}{}{t} = 2K/(t
+2K).
\end{gather*}
\vspace{-9mm}
\item \emph{Communication step:} user $k$ sends its updated model $\alpha_k^{\mtime{t}}$ to its neighborhood $N_k$.
\end{enumerate}

Importantly, the above update 
only requires the knowledge of the models of neighboring users, which were
received at earlier iterations.

%% file: analysis.tex
\subsection{Convergence Analysis, Communication and Memory Costs}
\label{sec:convergence}

\newcommand{\obj}{J}

The convergence analysis of our algorithm essentially follows the proof
technique proposed
in \citep{pmlr-v28-jaggi13} and refined in \citep{pmlr-v28-lacoste-julien13}
for the case of block coordinate Frank-Wolfe.
It is based on defining a surrogate for the optimality
gap $f(\alpha) - f(\alpha^*)$, where $\alpha^*\in\argmin_{\alpha\in\mathcal{M}}f(\alpha)$. Under
an appropriate notion of smoothness for $f$ over the
feasible domain, the convergence is
established by showing that the gap decreases in expectation with the number of iterations, because at a given iteration $t$ the block-wise surrogate gap
at the current solution is minimized by the greedy
update $\kt{s}{k}{t}$. We obtain that our algorithm achieves
an $O(1/t)$ convergence rate (see supplementary for the proof).





\begin{theorem}\label{the:optimal}
Our decentralized Frank-Wolfe algorithm takes at most $6 K (C^\otimes_f +
p_0)/\varepsilon$ iterations to find an approximate solution $\alpha$ that
satisfies, in expectation, $f(\alpha) - f(\alpha^*) \leq \varepsilon$, where
$C^\otimes_f\leq4 \beta^2 \sum_{k=1}^K d_k(w)( c_k \|A_k\|^2  + \mu_1)$ and $p_0
= f(\kt{\alpha}{}{0}) - f(\alpha^*)$ is the initial sub-optimality gap.
\end{theorem}

Theorem~\ref{the:optimal} shows that large degrees for users with low
confidence and small margins
penalize the convergence rate much less than for users with large confidence and large margins.
This is rather intuitive as users in the latter case have greater influence on the overall solution in Eq.~\eqref{eq:general-joint-obj}.

Remarkably, using a few tricks in the representation of the sparse updates,
the communication and memory cost needed by our algorithm to
converge to an $\epsilon$-approximate solution can be shown to be linear in
the number of edges of the graph and \emph{logarithmic} in the number of base
predictors. We refer to the supplementary material for details. For the
classic case where base predictors consist of a constant number of decisions
stumps per feature, this translates into a logarithmic
cost in the \emph{dimensionality of the data} leading to significantly better
complexities than the state-of-the-art (see the experiments of Section~
\ref{sec:exp}).  

\begin{remark}[Other loss functions]
We focus on the Adaboost log loss \eqref{eq:adaboost} to
emphasize that we can learn nonlinear models while keeping the formulation
convex. We point out that our algorithm and analysis readily extend to other
convex loss functions, as long as we keep an L1-constraint on the parameters.
\end{remark}




%% file: discovery.tex

\section{DECENTRALIZED LEARNING OF COLLABORATION GRAPH}
\label{sec:discovery}



In the previous section, we have proposed and analyzed an algorithm to learn
the model parameters $\alpha$ given a fixed collaboration graph $w$. To make
our fully decentralized alternating optimization scheme complete, we now turn
to the converse problem
of optimizing the graph weights $w$ given fixed models $\alpha$.
We will work with flexible graph regularizers $g(w)$ that are
\emph{weight and
degree-separable}:
\begin{equation*}
g(w) = \textstyle\sum_{k < l}g_{k,l}(w_{k,l}) + \sum_{k=1}^K g_k(d_k(w)),
\end{equation*}
where $g_{k,l}:\mathbb{R}\rightarrow\mathbb{R}$ and $g_k:
\mathbb{R}\rightarrow\mathbb{R}$ are convex and smooth. This
generic form
allows to regularize weights and degrees in a flexible way (which
encompasses some recent work from the graph signal processing community 
\citep{Dong2016,kalo,Berger2018a}), while the separable structure is key to the
design of an efficient decentralized algorithm that relies only on local communication.
We denote the graph learning objective function by $h(w):=J
(\alpha,w)$ for fixed models $\alpha$.
Note that $h(w)$ is convex in $w$.

\textbf{Decentralized algorithm.}
Our goal is to design a fully decentralized algorithm to update the
collaboration graph $\mathcal{G}_w$. We thus need users to communicate beyond
their current direct neighbors in $\mathcal{G}_w$ to discover new relevant
neighbors. In order to preserve scalability to large numbers of users, a user
can only communicate with small random batches of other users. In a
decentralized system, this can be implemented by a
classic primitive known as a peer sampling service \citep{peersampling,5738983}.
Let $\kappa\in [1..K-1]$ be a parameter of the algorithm, which in
practice is much smaller than $K$.
At each step, a random user $k$ wakes up and samples uniformly and
without replacement a set $\mathcal{K}$ of $\kappa$ users from the set $
\{1,\dots, K\} \setminus \{k\}$ using the peer sampling service. We denote by $w_{k, \mathcal{K}}$ the
$\kappa$-dimensional subvector of a vector $w\in\mathbb{R}^{K(K-1)/2}$ corresponding to the entries $\{(k,l)\}_
{l\in\mathcal{K}}$. Let $\Delta_{k, \mathcal{K}} = 
(\|\alpha_k-\alpha_l\|^2)_{l\in\mathcal{K}}$
, $p_{k, \mathcal{K}} = (c_k\mathcal{L}_k(\alpha_k; S_k) + c_l
\mathcal{L}_l(\alpha_l; S_l))_{l\in\mathcal{K}}$ and
$v_{k, \mathcal{K}}(w) = (
g_k'(d_k(w)) +g_l'(d_l(w)) + g'_{k,l}(w_{k,l}))_{l\in\mathcal{K}}$.
The partial derivative of the objective $h(w)$ with respect to the
variables $w_{k, \mathcal{K}}$ can be written as follows:
\begin{equation}
\label{eq:gradpartial}
[\nabla h(w)]_{k, \mathcal{K}} = p_{k, \mathcal{K}} + (\mu_1/2)\Delta_{k, 
\mathcal{K}} + \mu_2 v_{k, \mathcal{K}}(w).
\end{equation}

We denote by $L_{k,\mathcal{K}}$
is the Lipschitz constant of $\nabla h$ with respect to block $w_{k, 
\mathcal{K}}$. We now state our algorithm. We start from some arbitrary weight
vector $w^
{\mtime{0}}\in\mathcal{W}$, each user having a local copy of its $K-1$
weights. At each
time step $t$, a random user $k$ wakes up and
performs the following actions:
\begin{enumerate}
\item Draw a set $\mathcal{K}$ of $\kappa$ users and request their current
models, loss value and degree.
\vspace{-2mm}
\item Update the associated weights:
\vspace{-3mm}
$$w^{\mtime{t+1}}_{k, \mathcal{K}} \leftarrow \max\big(0, w^{\mtime{t}}_{k, \mathcal{K}} - 
(1/L_{k,\mathcal{K}})[\nabla h(w^{\mtime{t}})]_{k, \mathcal{K}}\big).$$
\vspace{-8mm}
\item Send each updated weight $w^{\mtime{t+1}}_{k,l}$ to the associated user in $l\in\mathcal{K}$.
\end{enumerate}
The algorithm is fully decentralized. Indeed, no global information is needed
to update the weights: the information requested from users in $
\mathcal{K}$ at step 1 of the algorithm is sufficient to compute 
\eqref{eq:gradpartial}. Updates can thus happen asynchronously and in
parallel.


\textbf{Convergence, communication and memory.}
Our analysis proceeds as follows. We first show that our
algorithm
can be seen as an instance of proximal coordinate
descent (PCD) \citep{Tseng_Yun09,proxcd} on a slightly modified objective
function. Unlike
the standard PCD setting which focuses on disjoint blocks, our coordinate
blocks exhibit a specific overlapping structure that arises as soon as $\kappa
> 1$ (as each weight is shared by two users). We
build upon the PCD analysis due to
\citep{wright2015coordinate}, which we adapt to account for
our overlapping block structure. The details of our analysis can be found in the
supplementary material. For the case where $g$ is strongly convex, we obtain
the following convergence rate.\footnote{For the general convex case, we can
obtain a slower $O(1/T)$ convergence rate.}

\begin{theorem}
\label{the:disco}
Assume that $g(w)$ is $\sigma$-strongly convex. Let $T>0$ and $h^*$ be the
optimal
objective value. Our algorithm cuts the expected suboptimality gap
by a constant factor $\rho$ at each iteration: we have $\mathbb{E}[h(w^{(T)}) -
h^*]
\leq \rho^T(h(w^{\mtime{0}}) -
h^*)$ with $\rho = 1- \frac{2\kappa\sigma}{K
(K-1)L_{max}}$ with $L_{max}=\max_{(k,\mathcal{K})}L_{k,\mathcal{K}}$.
\end{theorem}

The rate of Theorem~\ref{the:disco} is typically faster
than the sublinear rate of the boosting subproblem (Theorem~\ref{the:optimal}), suggesting that a small
number of updates per user is sufficient to reach reasonable optimization
error before re-updating the models given the new graph. In the supplementary,
we further analyze the trade-off between communication and memory costs and
the convergence rate ruled by $\kappa$.

\textbf{Proposed regularizer.}
In our
experiments, we use a graph regularizer defined as $g(w) = \lambda\|w\|^2 - 
\mathbf{1}^\top\log(d(w) + \delta)$, which is inspired from \citep{kalo}.
The log term ensures that all nodes have
nonzero degrees (the small positive constant $\delta$ is a simple trick to
make the logarithm smooth on the feasible domain, see e.g., \citep{koriche})
without ruling out non-connected graphs with several connected components.
Crucially, $\lambda > 0$ provides a direct way to tune the sparsity of
the graph: the smaller $\lambda$, the more concentrated the weights of a
given user on the peers with the closest models. This
allows us to control the trade-off between accuracy and communication in the
model update step of Section~\ref{sec:method}, whose communication cost is
 linear in the number of edges. 
The resulting objective is strongly convex and block-Lipschitz
continuous (see supplementary for the derivation of the parameters and
analysis of the trade-offs).
Finally, as discussed in \citep{kalo}, tuning the importance
of the log-degree term with respect to the other graph terms has simply a
scaling effect, thus we can simply set $\mu_2=\mu_1$ in 
\eqref{eq:general-joint-obj}.

\begin{remark}[Reducing the number of variables]
To reduce the number of variables to optimize, each user can keep to 0
the weights corresponding to users whose current model is most different to
theirs. This heuristic has a negligible impact on the solution quality in
sparse regimes (small $\lambda$).
\end{remark}


%% file: experiments.tex

\section{EXPERIMENTS}
\label{sec:exp}

In this section, we study the practical behavior of our approach. Denoting our
decentralized Adaboost method introduced in Section~\ref{sec:method} as
\algoname, we study two variants: \algonameF (which uses a fixed oracle graph
given as input) and \algonameL (where the graph is learned along with the
models).
We compare against various competitors, which learn either
global or personalized models in a centralized or decentralized manner.
\globalboost and \globallin learn a single global
$l_1$-Adaboost model (resp. linear model) over the centralized dataset
$S=\cup_{k}S_k$.
\localboost and \locallin learn (Adaboost or linear) personalized models
independently for each user without collaboration. Finally, \persolin is a
decentralized method
for collaboratively learning personalized linear models 
\citep{vanhaesebrouck2016decentralized}. This approach requires an oracle
graph as input (\persolinF) but it can also directly benefit
from our graph learning approach of Section~\ref{sec:discovery} (we denote
this new variant by \persolinL).
We use the same set
of base predictors for all boosting-based methods,
namely $n$ simple decision stumps uniformly split between all $D$ dimensions
and value ranges.
For all methods we tune the hyper-parameters with 3-fold cross validation.
Models are initialized to zero vectors and the initial graphs of \algonameL
and \persolinL are learned using the purely local classifiers, and then
updated after every $100$ iterations of optimizing the classifiers, with
$\kappa=5$. 
All reported accuracies are averaged over users.
Additional details and results can be found in the supplementary.
The source code is available at \url{https://github.com/vzantedeschi/Dada}.


\begin{figure}[htbp]
    \centering
    \subfigure
    {\label{fig:accuracies}
      \includegraphics[width=\linewidth]{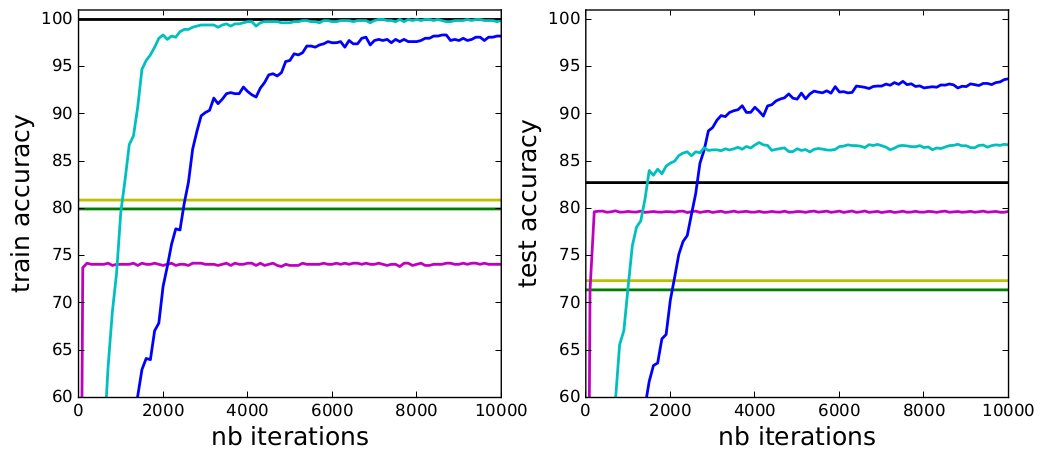}}
    \qquad
    \subfigure
    {\label{fig:edges} 
      \includegraphics[width=.55\linewidth]{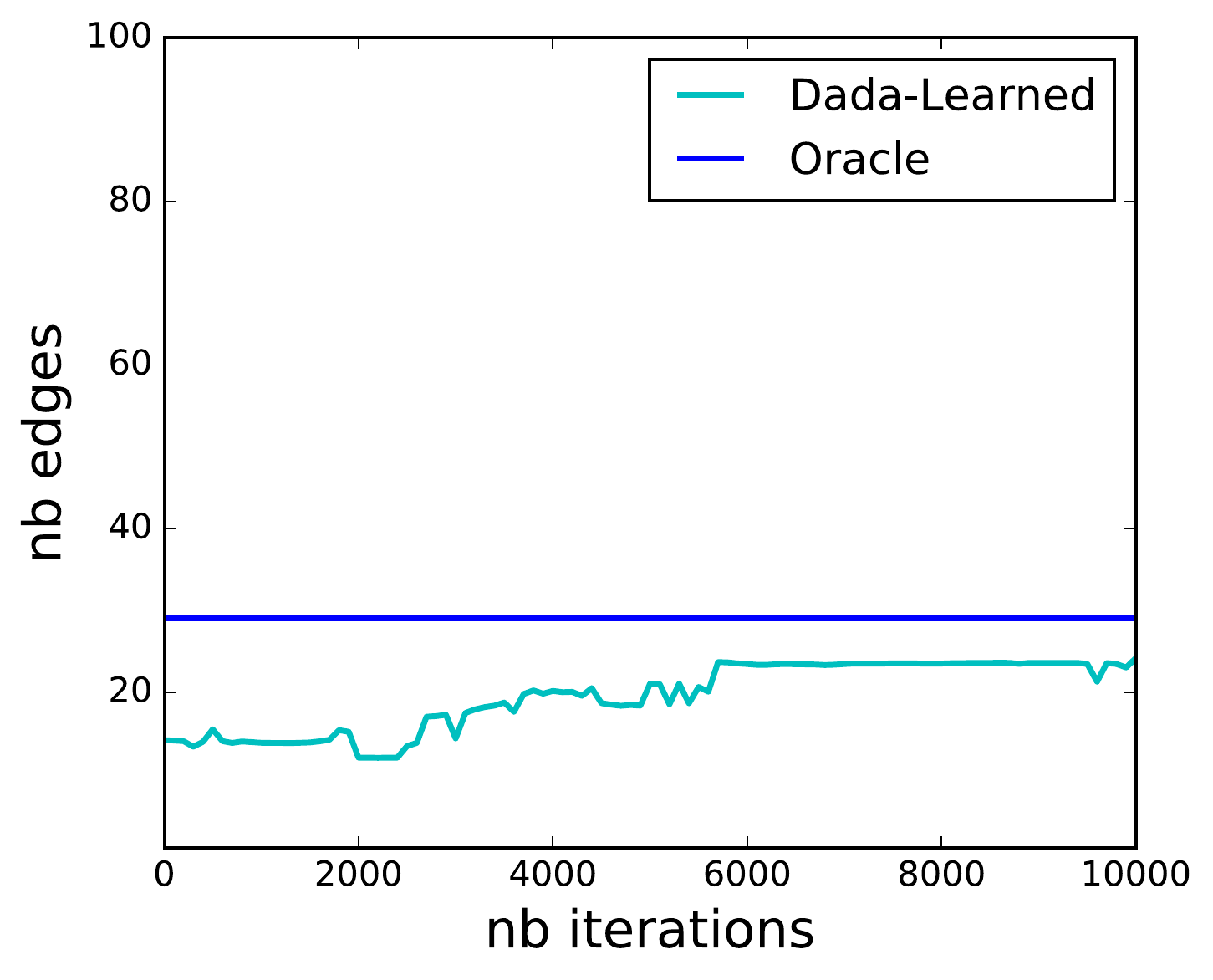}
          \raisebox{1cm}{\includegraphics[height=2cm]{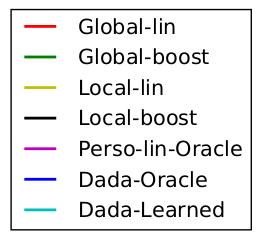}}}
    \caption{Results on the Moons dataset. \emph{Top:} Training and test
    accuracy \wrt iterations (we display the performance of
    non-collaborative baselines at convergence with a straight line).
    \globallin is off limits at $\sim$50\% accuracy. \emph{Bottom:} Average
    number of
    neighbors \wrt
    iterations for \algonameL.}
  \label{fig:moons}
\end{figure}

\begin{table}
  \caption{Test accuracy (\%) on real data, averaged over 3 runs. Best
  results in boldface, second best in italic.}
  \label{tab:acc}
  \centering
  ~\\\vspace{1mm}
  \footnotesize
  \scalebox{0.85}{
  \begin{tabular}{lcccc}
    \bf{DATASET}      & {\sc \bf{HARWS}}  & {\sc \bf{VEH.}} & {\sc \bf{COMP.}} & {\sc \bf{SCH.}} \\
    \midrule
    Global-linear & 93.64 & 87.11 & 62.18 & 57.06 \\
    Local-linear & 92.69 & 90.38 & 60.68 & 70.43\\
    Perso-linear-Learned  & \first{96.87} & \first{91.45} & 69.10 & \second{71.78}\\
    Global-Adaboost & 94.34 & 88.02 & \second{69.16 }& 69.96 \\
    Local-Adaboost     & 93.16 & 90.59 & 66.61 & 70.69 \\
    Dada-Learned & \second{95.57} & \second{91.04} & \first{73.55} & \first{72.47}\\
  \end{tabular}}
\end{table}

\begin{table*}[!h]
  \caption{Test accuracy (\%) with different fixed communication budgets (\# bits) on real datasets.}
  \label{tab:bud}
  ~\\\vspace{1mm}
  \centering
  \footnotesize
  \scalebox{0.85}{
  \begin{tabular}{llcccc}
    \bf{BUDGET} & \bf{MODEL} & {\sc \bf{HARWS}} & {\sc \bf{VEHICLE}} & {\sc \bf{COMPUTER}} & {\sc
    \bf{SCHOOL}} \\
    \midrule
    \multirow{2}{*}{$DZ \times 160$} & \persolinL & - & - & - & -\\
                            & \algonameL & \bf{95.70} & \bf{75.11} & \bf{52.03} & 
                            \bf{56.83}\\
    \hline
    \multirow{2}{*}{$DZ \times 500$} & \persolinL & 81.06 & \bf{89.82} & - &
    -\\
                            & \algonameL & \bf{95.70} & 89.57 & {\bf 62.22} & \bf{71.90}\\
    \hline
    \multirow{2}{*}{$DZ \times 1000$} & \persolinL & 87.55 & 90.52 & \bf{68.95} & 71.90\\
                            & \algonameL & \bf{95.70} & \bf{90.81} & 68.83 & \bf{72.22}\\                        
  \end{tabular}}
\end{table*}

\textbf{Synthetic data.}
To study the behavior of our approach in a controlled setting, our
first set
of experiments is carried out on a synthetic problem ({\sc Moons}) constructed
from the classic two interleaving Moons dataset which has nonlinear class
boundaries.
We consider $K=100$ users, clustered in $4$ groups of $10$, $20$, $30$ and
$40$ users. Users in the same cluster are associated with a similar
rotation of the feature space and hence have similar tasks. We construct an
oracle collaboration graph based on the difference in
rotation angles between users, which is given as input to \algonameF and
\persolinF.
Each user $k$ obtains a
training sample random size $m_k\sim\mathcal{U}(3,15)$.
The data dimension is $D=20$ and the
number of base predictors is $n=200$.
We refer to the supplementary material for more details on the dataset
generation.
Figure~\ref{fig:moons} (left) shows the accuracy of all methods. As
expected, all linear models (including \persolin) perform poorly since the tasks
have highly
nonlinear decision boundaries. The results show the
clear gain in accuracy provided by our method: both \algonameF and \algonameL
are successful in reducing the overfitting of \localboost, and also achieve
higher test accuracy than \globalboost. \algonameF
outperforms \algonameL as it makes use of the oracle graph computed from the
true data distributions. Despite
the noise introduced by the finite sample setting, \algonameL effectively
makes up for not having access to any knowledge on the relations
between the users' tasks.
Figure~\ref{fig:moons} (right) shows that
the graph learned by \algonameL remains sparse across time (in fact, always
sparser than the oracle graph), ensuring a small communication cost for
the model update steps. 
Figure~\ref{fig:graph}
(left)
confirms that
the graph learned by \algonameL is able to approximately recover the
ground-truth cluster structure. Figure~\ref{fig:graph} (right) provides a more
detailed visualization of the learned graph. We can clearly see the effect of the inductive bias
brought by the confidence-weighted loss term in Problem~\eqref{eq:general-joint-obj}
discussed
in Section~\ref{sec:setting}. In
particular, nodes
with high confidence and high loss values tend to have small degrees while
nodes with low confidence or low loss values are more densely connected.

\begin{figure}[h]
    \centering
    \includegraphics[width=.45\linewidth]
    {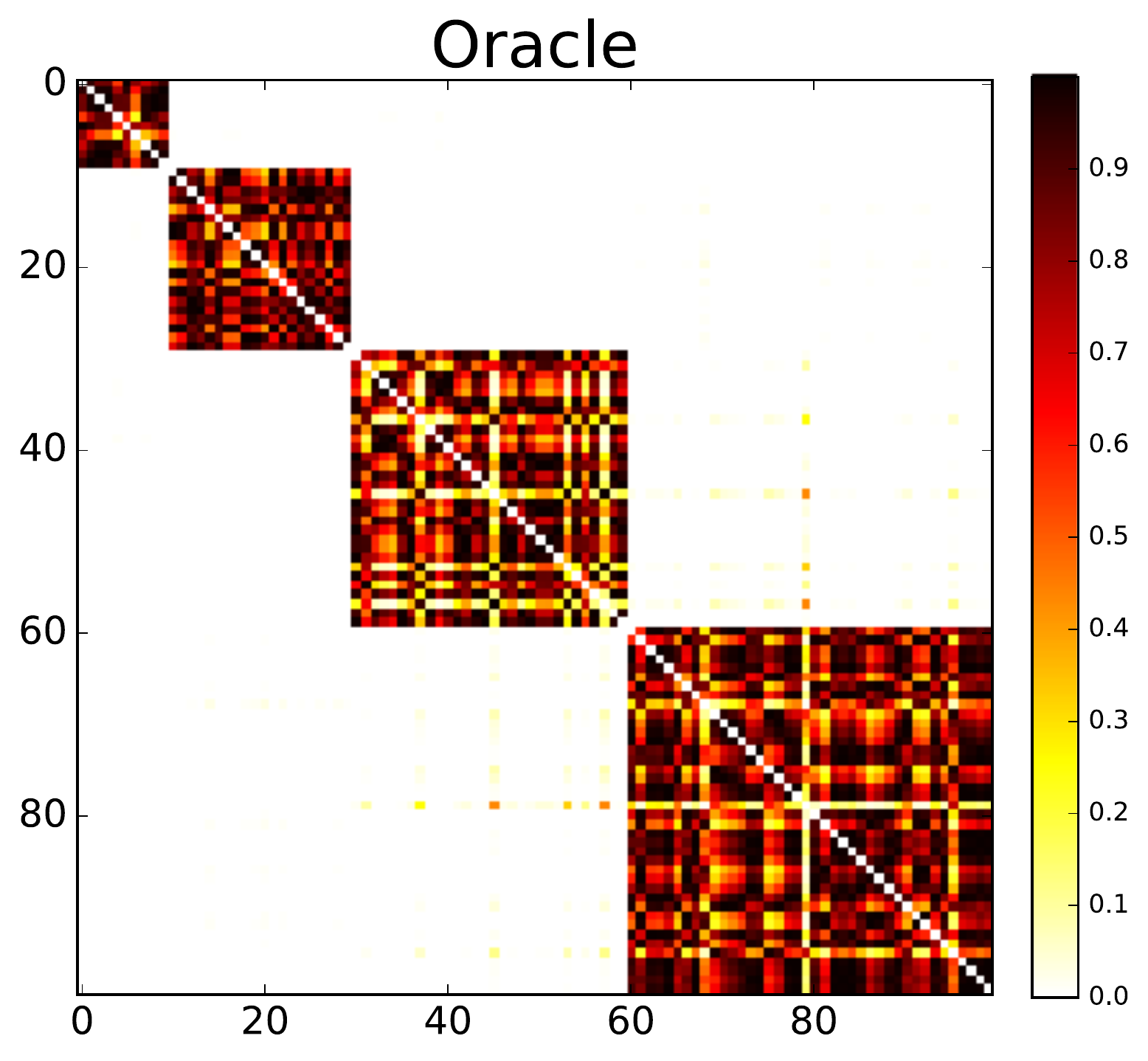}
    \hspace{5mm}
    \includegraphics[width=.45\linewidth]
    {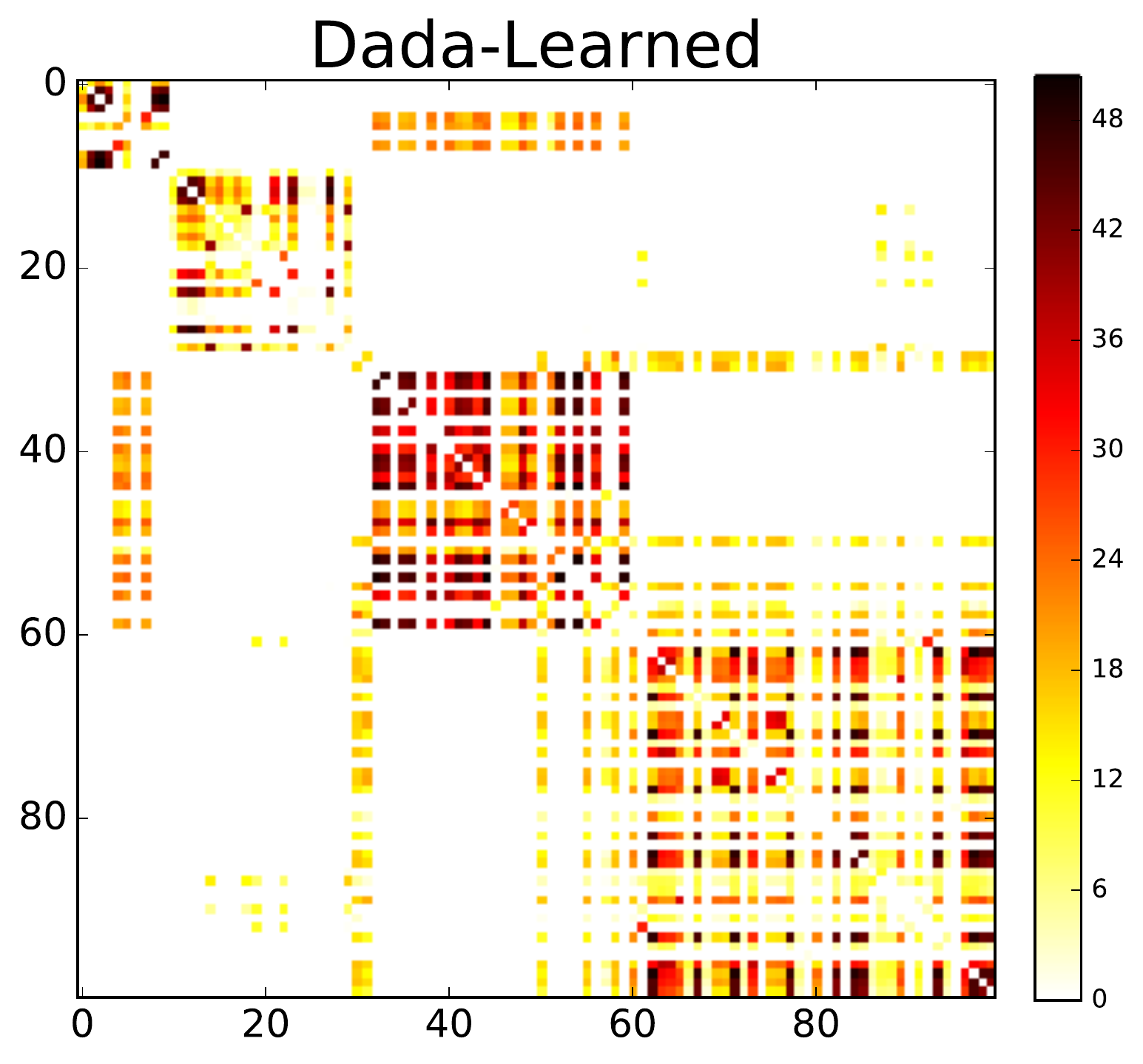}
    \\
    \includegraphics[width=.93\linewidth]
    {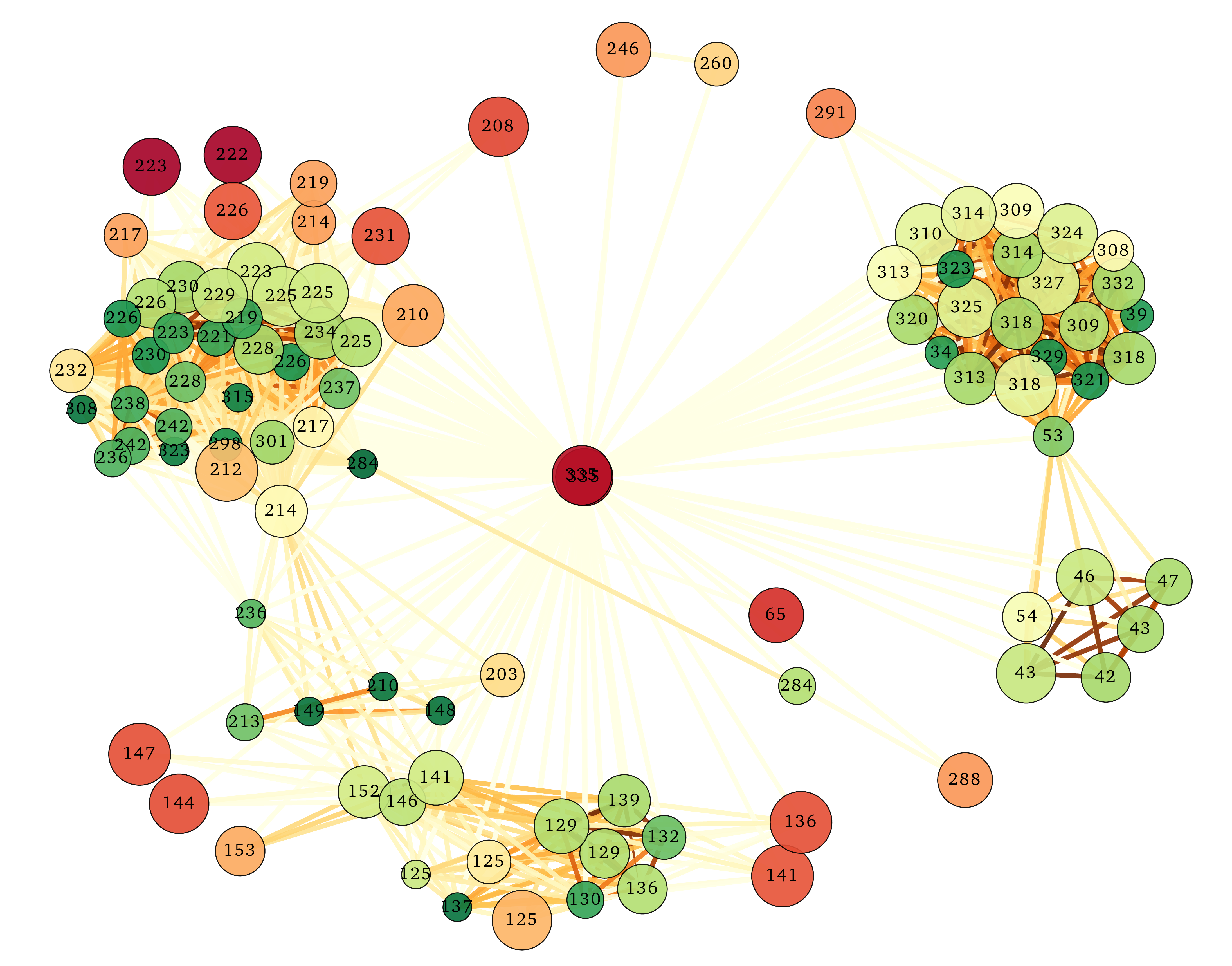}
    \caption{Graph learned on {\sc Moons}. \emph{Top:} Graph weights for the
    oracle and learned graph (with users grouped by cluster).
  \emph{Bottom:} Visualization of the graph.
    The node size is proportional to the confidence $c_k$ and the color
    reflects the
    relative value of
    the local
    loss (greener = smaller loss).
    Nodes are labeled with their rotation angle, and a darker edge
    color indicates a higher weight.
    } 
    \label{fig:graph}
\end{figure}

\textbf{Real data.}
We present results on real datasets that are naturally
collected at the user level: Human Activity
Recognition With Smartphones ({\sc Harws}, $K=30$, $D=561$)~
\citep{anguita2013public}, {\sc
Vehicle Sensor}~\citep{duarte2004vehicle} ($K=23$, $D=100$), {\sc Computer
Buyers} ($K=190$, $D=14$) and {\sc
School}~\citep{goldstein1991multilevel} ($K=140$, $D=17$).
As shown in Table~\ref{tab:acc}, \algonameL and \persolinL, which both make
use of our alternating procedure, achieve the best performance. This
demonstrates the wide applicability of our graph learning approach, for it
enables the use of \persolin \citep{vanhaesebrouck2016decentralized}
on datasets where no prior information is available to build a predefined
collaboration graph.
Thanks to its logarithmic communication, our approach \algonameL
achieves higher accuracy under limited communication budgets, especially on
higher-dimensional
data (Table~\ref{tab:bud}). More details and results are given in the
supplementary.



%% file: conclusion.tex

\section{FUTURE WORK}
\label{sec:conclu}

We plan to extend our approach
to (functional) gradient boosting \citep{gradboosting,wang2015functional}
where the graph regularization term would need to be applied to an infinite
set of base predictors.
Another promising direction is to make our approach
differentially-private \citep{Dwork2006a} to formally guarantee that
personal datasets cannot be inferred from the information
sent by users.
As our algorithm communicates very scarcely,
we think that the privacy/accuracy trade-off may be better than the one
known for linear models \citep{bellet2017fast}.




%% file: acks.tex
\paragraph{Acknowledgments}
~\\

The authors would like to thank Rémi Gilleron for
his useful feedback. This research was partially sup-
ported by grants ANR-16-CE23-0016-01 and ANR-15-
CE23-0026-03, by the European Union’s Horizon 2020
Research and Innovation Program under Grant Agree-
ment No. 825081 COMPRISE and by a grant from
CPER Nord-Pas de Calais/FEDER DATA Advanced
data science and technologies 2015-2020.

%% file: proof-fw.tex

\section{PROOF OF THEOREM~\ref{the:optimal}}
\label{app:fw}
We first recall our optimization problem over the classifiers
$\alpha=[\alpha_1,\dots,\alpha_K]\in(\mathbb{R}^n)^K$:
\begin{multline}
  \label{eq:obj_supp}
  \min_{\normonenormal{\alpha_1},\dots,\normonenormal{\alpha_K} \leq \beta} f(\alpha) = 
  \sum_{k=1}^{K} d_k(w) c_k \log\Big( \frac{1}{m_k}\sum_{i=1}^{m_k} \exp\big
  ( -(A_k\alpha_k)_i \big) \Big) + \frac{\mu_1}{2} \sum_{k=1}^{K} \sum_{l=1}^
  {k-1} w_{k,l}\|\alpha_k-\alpha_l\|^2.
\end{multline}

We recall some notations.
For any $k\in[K]$, we let $\mathcal{M}^k = \{\alpha_k\in\mathbb{R}^n : \normone{\alpha_k}\leq \beta\}$ and denote by $\mathcal{M} = \mathcal{M}^1\times\dots\times\mathcal{M}^K$ our feasible domain in \eqref{eq:obj_supp}. 
We also denote by $v_{[k]}\in\mathcal{M}$ the zero-padding of any vector $v_k\in\mathcal{M}^k$. 
Finally, for conciseness of notations, for a given $ \gamma \in [ 0,1 ] $ we write $\hat{\alpha} = \alpha + \gamma (s_{[k]} - \alpha_{[k]})$ and $\hat{\alpha}_k = (1 - \gamma)\alpha_k + \gamma s_k$.

\subsection{Curvature Bound}

We first show that our objective function satisfies a form of smoothness over
the feasible domain, which is expressed by a notion of curvature.
Precisely, the global product curvature constant $C^\otimes_f$ of $f$ over $\mathcal{M}$ is the sum over each block of the maximum relative deviation of $f$ from its linear approximations over the block \citep{pmlr-v28-lacoste-julien13}:
  \begin{equation}
    \label{eq:global-curv}
    C^\otimes_f = \sum_{k=1}^K {C_f^k} = \sum_{k=1}^K \sup_{ \substack{ \alpha
    \in \mathcal{M}, s_k \in \mathcal{M}^k\\ \gamma\in[0,1]}} \Big\{ \frac{2}
    {\gamma^2} \left ( f(\hat{\alpha}) - f(\alpha) - (\hat{\alpha}_k -
    \alpha_k)^\top \nabla_k f(\alpha) \right ) \Big\}.
  \end{equation}
We will use the fact that each partial curvature constant $C_f^k$ is upper
bounded by the (block) Lipschitz constant of the partial gradient $\nabla_k f
(\alpha)$ times the squared diameter of the blockwise feasible domain $
\mathcal{M}^k$ \citep{pmlr-v28-lacoste-julien13}. The next lemma gives a bound on the product space curvature $C^\otimes_f$.

\begin{lemma}
    \label{lemma:curvature}
    For Problem~\eqref{eq:obj_supp}, we have $C^\otimes_f\leq4 \beta^2 \sum_
    {k=1}^K d_k(w)( c_k \normone{A_k}^2  + \mu_1)$.
\end{lemma}

\begin{proof}
    For the following proof, we rely on two key concepts: the Lipschitz continuity and the diameter of a compact space.
    A function $f: \mathcal{X} \to \mathbb{R}$ is $L$-lipschitz w.r.t. the norm $ \normone{.} $ if $\forall (x, x') \in \mathcal{X}^2$:
    \begin{equation}
        |f(x) - f(x')| \leq L \normone{x - x'}.
    \end{equation}
    The diameter of a compact normed vector space $(\mathcal{M}, \norm{.})$ is defined as:
    \begin{equation}
        \textstyle\diam_{\norm{.}}(\mathcal{M}) = \displaystyle\sup_{x, x' \in \mathcal{M}} \norm{x - x'}.
        \nonumber
    \end{equation}
    We can easily bound the diameter of the subspace $\mathcal{M}^k = \{\alpha_k\in\mathbb{R}^n : \normone{\alpha_k}\leq \beta\}$ as follows:
    \begin{equation}
      \label{eq:diam}
      \textstyle\diam_{\normonenormal{.}}(\mathcal{M}^k) = \max_{\alpha_k, \alpha'_k \in \mathcal{M}^k} \normone{\alpha_k - \alpha'_k} = 2 \beta .
    \end{equation}
    We recall the expression for the partial gradient:
    \begin{equation}
    \label{eq:grad2}
        \nabla_k f(\alpha) = - d_k(w) c_k \adaweight_k(\alpha)^\top A_k + \mu_1
        \Big(d_k(w) \pad{\alpha}{k} - \sum_l w_{k,l}\pad{\alpha}{l} \Big),
    \end{equation}
    where we denote $\adaweight_k(\alpha) = \frac{\exp(-A_k \pad{\alpha}{k})}{ \sum_{i=1}^{m_k} \exp(-A_k \pad{\alpha}{k})_i}$. We bound the Lipschitz constant of $\adaweight_k(\alpha)$ by bounding its first derivative:
    \begin{align}
        \normone{\nabla_k (\adaweight_k(\alpha))} &= \normone{(- \adaweight_k
        (\alpha) + \adaweight_k(\alpha)^2)^\top A_k} \nonumber \\
        & \leq \normone{A_k}.\label{lip:w}
    \end{align}
    Eq. \eqref{lip:w} is due to the fact that $ \normone{\adaweight_k(\alpha)}
    \leq 1 $ and $ \adaweight_k(\alpha) \geq 0$. It is then easy to see that
    considering any two vectors $\alpha, \alpha' \in \mathcal{M}$ differing only in their $k$-th block ($\pad{\alpha}{l} = \pad{\alpha'}{l}~ \forall l \ne k$), the Lipschitz constant $L_k$ of the partial gradient $\nabla_k f$ in \eqref{eq:grad2} is bounded by $d_k(w) (c_k \|A_k\|_1^2 + \mu_1)$.









    Finally, we obtain Lemma~\ref{lemma:curvature} by combining the above results \eqref{eq:diam} and \eqref{lip:w}:
    \begin{equation}
        C^\otimes_f = \sum_{k=1}^K C_f^k \leq \sum_{k=1}^K L_k \textstyle\diam^2_{\normonenormal{.}}(\mathcal{M}^k)
        \leq 4 \beta^2 \sum_{k=1}^K d_k(w)( c_k \normone{A_k}^2  + \mu_1).
    \end{equation}
\end{proof}

\subsection{Convergence Analysis}

We can now prove the convergence rate of our algorithm by following the proof
technique proposed by Jaggi in~\citep{pmlr-v28-jaggi13} and refined by \citep{pmlr-v28-lacoste-julien13} for the case of block coordinate Frank-Wolfe.

We start by introducing some useful notation related to our problem 
\eqref{eq:obj_supp}:
\begin{align}
\gap(\alpha) &= \max_{s\in \mathcal M}\{ (\alpha-s)^\top \nabla(f(\alpha))\}
 = \textstyle\sum_{k=1}^K \gap_k(\alpha_k)\nonumber\\
&= \textstyle\sum_{k=1}^K \max_{s_k \in \mathcal{M}^k} \{ (\alpha_k -
s_k)^\top
\nabla_k f(\alpha) \} \label{eq:dualgap}
\end{align}
The quantity $\gap(\alpha)$ can serve as a certificate for the quality of a
current approximation of the
optimum of the objective function \citep{pmlr-v28-jaggi13}. In particular, one can show that $f
(\alpha)-f(\alpha^*) \leq \gap(\alpha)$ where $\alpha^*$ is a solution of 
\eqref{eq:obj_supp}. Under a bounded global product curvature constant
$C^\otimes_f$, we will obtain the convergence of Frank-Wolfe by showing that
the surrogate gap decreases in expectation over the
iterations, because at a given iteration $t$ the block-wise surrogate gap at
the current solution $\gap_k(\kt{\alpha}{k}{t})$ is minimized by the greedy
update $\kt{s}{k}{t} \in \mathcal{M}^k$.

  Using the definition of the curvature \eqref{eq:global-curv} and rewriting $\hat{\alpha}_k - \alpha_k = -\gamma_k(\alpha_k-s_k)$, we obtain  
  \begin{equation*}
    f(\hat{\alpha}) \leq f(\alpha) - \gamma   (\alpha_k-s_k)^\top
    \nabla_k f(\alpha) + \gamma^2 \frac{{C_f^k}}{2}.
  \end{equation*}
  In particular, at any iteration $t$, the previous inequality holds for
  $\gamma = \kt{\gamma}{}{t} = \frac{2K}{t+2K}$, $\kt{\alpha}{}{t+1} = 
  \kt{\alpha}{}{t} + \kt{\gamma}{}{t} ( \kt{s}{k}{t+1} - \kt{\alpha}{k}{t+1})
  $ with $\kt{s}{k}{t+1} = \argmin_{s \in \mathcal{M}^k} \{s^\top \nabla_k f
  (\kt{\alpha}{}{t})\}$ as defined in \eqref{eq:fwupdate}. Therefore, 
  $ (\alpha_k-s_k)^\top \nabla_k f(\alpha)$ is by definition $\gap_k
  (\alpha_k)$ and 
  \begin{equation*}
    f(\kt{\alpha}{}{t+1}) \leq f(\kt{\alpha}{}{t}) - \kt{\gamma}{}{t} \gap_k(\kt{\alpha}{k}{t}) + (\kt{\gamma}{}{t})^2 \frac{{C_f^k}}{2} \enspace .
  \end{equation*}
  By taking the expectation over the random choice of $k \sim \mathcal{U}(1, K)$ on both sides, we obtain 
  \begin{align}
   \shortexpect{f(\kt{\alpha}{}{t+1})} 
   & \leq \shortexpect{f(\kt{\alpha}{}{t})} - \kt{\gamma}{}{t} \shortexpect{\gap_k(\kt{\alpha}{k}{t})} + \frac{(\kt{\gamma}{}{t})^2 \shortexpect{{C_f^k}}}{2} \nonumber \\
   & \leq \shortexpect{f(\kt{\alpha}{}{t})} - \frac{\kt{\gamma}{}{t} \gap(\kt{\alpha}{}{t})}{K} + \frac{(\kt{\gamma}{}{t})^2 C^\otimes_f}{2 K} \enspace . \label{eq:expe}
  \end{align}
  Let us define the sub-optimality gap $p(\alpha) = f(\alpha) - f^*$ with $f^*$ the optimal value of $f$. 
  By subtracting $f^*$ from both sides in~\eqref{eq:expe}, we obtain
  \begin{align}
    \shortexpect{p(\kt{\alpha}{}{t+1})} &\leq \shortexpect{p(\kt{\alpha}{}{t})} - \frac{\kt{\gamma}{}{t}}{K} \shortexpect{p(\kt{\alpha}{}{t})} + (\kt{\gamma}{}{t})^2 \frac{C^\otimes_f}{2K} \label{eq:leq-gap} \\
    & \leq \left(1 - \frac{\kt{\gamma}{}{t}}{K} \right) \shortexpect{p(\kt{\alpha}{}{t})} + (\kt{\gamma}{}{t})^2 \frac{C^\otimes_f}{2K}.
  \end{align}
  Inequality \eqref{eq:leq-gap} comes from the definition of the surrogate gap~\eqref{eq:dualgap} which ensures that $ \shortexpect{p(\alpha)} \leq \gap(\alpha)$. 

  Therefore, we can show by induction that the expected sub-optimality gap satisfies $\shortexpect{p(\kt{\alpha}{}{t+1})} \leq \frac{2 K (C^\otimes_f + p_0)}{t + 2 K}$, with $p_0 = p(\kt{\alpha}{}{0})$ the initial gap.
  This shows that the expected sub-optimality gap $\shortexpect{p(\alpha)}$ decreases with the number of iterations with a rate $O(\frac{1}{t})$, which implies the convergence of our algorithm to the optimal solution. 
  The final convergence rate can then be obtained by the same proof as \citep{pmlr-v28-lacoste-julien13} (Appendix C.3 therein) combined with Lemma~\ref{lemma:curvature}.


%% file: proof-cd.tex

\section{PROOF OF THEOREM~\ref{the:disco}}
\label{app:cd}

We first show that our algorithm can be explicitly formulated as an instance
of proximal block coordinate descent (Section~\ref{sec:prox}).
Building upon this formulation, we prove the convergence rate in Section~
\ref{sec:conv_disco}.

\subsection{Interpretation as Proximal Coordinate Descent}
\label{sec:prox}

First, we reformulate our graph learning subproblem as an equivalent
unconstrained optimization problem by incorporating the nonnegativity
constraints into the objective:
\begin{align}
\label{eq:obj_uncons}
\min_{w\in\mathbb{R}^{K(K-1)/2}} & ~F(w) = h(w) + r(w),
\end{align}
where 
\begin{align}
  h(w) &= \sum_{k=1}^{K} d_k(w)c_k\mathcal{L}_k(\alpha_k; S_k) +
         \frac{\mu_1}{2}\sum_{k<l} w_{k,l}\|\alpha_k-\alpha_l\|^2 +\mu_2g(w)\\
  r(w) &= \sum_{k<l}\mathbb{I}_{\geq 0}(w_{k,l}).
\end{align}
In the expression above, $\mathbb{I}_{\geq 0}$ denotes the characteristic
function of the nonnegative orthant of $\mathbb{R}$: $\mathbb{I}_{\geq 0}
(x)=0$ if $x\geq 0$ and $+\infty$ otherwise. Recall that $g$ is the sum of smooth, weight-separable and degree-separable functions

$$g(w) = \textstyle\sum_{k < l}g_{k,l}(w_{k,l}) + \sum_{k=1}^K g_k(d_k(w))\enspace .$$

We assume that  $h(w)$ is strongly
convex and smooth, while it is clear that $r(w)$ is
not smooth but convex and separable across the coordinates of $w$. We will
denote by $w^*$ the solution to \eqref{eq:obj_uncons}, which is also the
solution to our original (constrained) graph learning subproblem.

We will now show that the algorithm presented in the main text can be
explicitly 
reformulated as an instance of proximal block coordinate descent 
\citep{proxcd} applied to the function $F
(w)$. In
the process, we will introduce some notations that we will reuse in
the convergence analysis provided in Section~\ref{sec:conv_disco}.
At each iteration $t$, a random block of coordinates indexed by $(k, 
\mathcal{K})$ is selected. Consider the following update:
\begin{equation}
\label{eq:rewrite_update1}
\left\{\begin{array}{l}
z^{(t)}_{k, \mathcal{K}} \leftarrow \displaystyle\argmin_{z\in
\mathbb{R}^\kappa}\Big\{ (z - w^{
(t)}_{k,
\mathcal{K}})^\top[\nabla h(w^{(t)})]_{k, \mathcal{K}} + \frac{L_{k,\mathcal{K}}}
{2}\|z
- w^{(t)}_{k,
\mathcal{K}}\|^2 + \sum_{l\in\mathcal{K}}\mathbb{I}_{\geq 0}(z_{k,l})\Big\}\\
w^{(t+1)} \leftarrow w^{(t)} + U_{k, \mathcal{K}}(z^{(t)}_{k, \mathcal{K}} - w^{(t)}_{k, 
\mathcal{K}})
\end{array}\right.
\end{equation}
For notational convenience, $U_{k, \mathcal{K}}$ denotes the column submatrix
of the $K(K-1)/2\times K(K-1)/2$ identity matrix such that $w^\top U_{k, 
\mathcal{K}} = w_{k, \mathcal{K}}\in\mathbb{R}^\kappa$ for any $w\in
\mathbb{R}^{K
(K-1)/2}$.

Notice that the minimization problem in \eqref{eq:rewrite_update1} is
separable and can be solved independently for each coordinate.
Denoting by
\begin{equation}
\label{eq:ztilde}
\tilde{z}^{(t)} = \argmin_{z\in
\mathbb{R}^{K(K-1)/2}}\big\{ (z - w^{
(t)})^\top\nabla h(w^{(t)}) + \frac{L_{k,\mathcal{K}}}{2}\|z - w^{(t)}\|^2 + r
(z)\big\},
\end{equation}
we can thus rewrite 
\eqref{eq:rewrite_update1} as:
\begin{equation}
\label{eq:rewrite_update2}
w^{(t+1)}_{j,l} = \left\{\begin{array}{ll}
\tilde{z}^{(t)}_{j,l} & \text{if } j=k \text{ and } l\in\mathcal{K}\\
w^{(t)}_{j,l} & \text{otherwise}
\end{array}\right.
\end{equation}

Finally, recalling the
definition of the proximal operator of a function $f$:
$${\prox}_{f}(w) = \argmin_{z} \Big\{f(z) + \frac{1}{2}\|w -
z\|^2\Big\},$$
we can rewrite \eqref{eq:ztilde} as:
\begin{equation}
\tilde{z}^{(t)} = {\prox}_{\frac{1}{L_{k,\mathcal{K}}}r}(w^{(t)} - (1 / L_{k,\mathcal{K}})\nabla
h(w^{
(t)})).
\end{equation}

We have indeed obtained that \eqref{eq:rewrite_update2} corresponds to a
proximal block coordinate descent update \citep{proxcd}, i.e. a proximal
gradient descent step restricted to a block of coordinates.

When $f$ is the characteristic function of a set, the proximal operator
corresponds to the Euclidean projection onto the set. Hence, in our case we
have ${\prox}_r(w) = \max(0, w)$ (the thresholding operator), and we recover
the simple update introduced in the main text.

\subsection{Convergence Analysis}
\label{sec:conv_disco}

We start by introducing a convenient lemma.

\begin{lemma}
\label{lem:blocksmooth}
For any block of size $\kappa$ indexed by $(k, \mathcal{K})$, any $w\in
\mathbb{R}^{K(K-1)/2}$
and any $z\in\mathbb{R}^\kappa$, we have:
$$h(w + U_{k,\mathcal{K}}z) \leq h(w) + z^\top[\nabla h(w)]_{k,\mathcal{K}} +
\frac{L_{k,\mathcal{K}}}{2}\|z\|^2.$$
\end{lemma}
\begin{proof}
This is obtained by applying Taylor's inequality to the function
\begin{align*}
q_w : & ~\mathbb{R}^\kappa\rightarrow\mathbb{R}\\
& ~z\mapsto h(w + U_{k,\mathcal{K}}z)
\end{align*}
combined with the convexity and $L_{k,\mathcal{K}}$-block smoothness of
$h$.
\end{proof}

We are now ready the prove the convergence rate of our algorithm. We focus
below on the more interesting cases where the block size $\kappa>1$, since the
case $\kappa=1$ (blocks of size 1) reduces to standard proximal
coordinate descent and can be addressed directly by previous work 
\citep{proxcd,wright2015coordinate}.

Recall that in our algorithm, at each iteration an user $k$ is drawn
uniformly at random from $[K]$, and then this user samples a set $
\mathcal{K}$ of $\kappa$ other users uniformly and without replacement from
the set $\{1, \dots, K\} \setminus \{k\}$. This gives rise to a block of
coordinates indexed by $(k, \mathcal{K})$. Let $\mathcal{B}$ be
the set of such possible block indices. Note that $\mathcal{B}$ has
cardinality $K{K-1\choose\kappa}$ since for $\kappa>1$ all ${K-1\choose\kappa}$ blocks that
can be sampled by an user are unique
(i.e., they cannot be sampled by other users).
However, it is important to
note that unlike commonly assumed in the block coordinate descent literature,
our blocks exhibit an overlapping structure: each coordinate block $b\in
\mathcal{B}$ shares some of its coordinates with several other
blocks in $\mathcal{B}$. In particular, each
coordinate (graph weight) $w_
{i,j}$ is shared by user $i$ and user $j$ can thus be part of
blocks drawn by both users.
Our analysis builds
upon the proof technique of
\citep{wright2015coordinate}, adapting the arguments to handle our update
structure based on overlapping blocks rather than single coordinates.

Let
$b_t=(k,
\mathcal{K})\in\mathcal{B}$ be the block of coordinates selected at iteration
$t$. For notational convenience, we write $(j,l)\in b_t$ to denote the set of
coordinates indexed by block $b_t$ (i.e., index pairs $(j,l)$ such that $j=k$
and $l\in\mathcal{K}$).
Consider the expectation of the objective function $F(w^{ (t+1)})$ in
\eqref{eq:obj_uncons} over the choice of $b_t$, plugging in the update 
\eqref{eq:rewrite_update1}:

\begin{align}
\mathbb{E}_{b_t}[F(w^{(t+1)})] &= \mathbb{E}_{b_t}\Big[h\big(w^{(t)} + U_{b_t}
(z^{(t)}_{b_t} - w^{(t)}_{b_t})\big) + \sum_{(j,l)\in b_t}\mathbb{I}_
{\geq 0}(z^{(t)}_
{j,l}) +
\sum_{
(j,l)\notin b_t}\mathbb{I}_{\geq 0}(w^{(t)}_{j,l})\Big]\nonumber\\
&= \frac{1}{K{K-1\choose \kappa}}\sum_{b\in\mathcal{B}}\Big[h\big(w^{(t)} + U_
{b}
(z^{(t)}_{b} - w^{(t)}_{b})\big) + \sum_{(j,l)\in b}\mathbb{I}_
{\geq 0}(z^{(t)}_
{j,l}) +
\sum_{
(j,l)\notin b}\mathbb{I}_{\geq 0}(w^{(t)}_{j,l})\Big]\nonumber\\
&\leq \frac{1}{K{K-1\choose \kappa}}\sum_{b\in\mathcal{B}} \Big[h(w^{(t)}) + (
z^{(t)}_{b}
-
w^{(t)}_{b})^\top[\nabla h(w^{(t)})]_b + \frac{L_b}{2}\|z^{(t)}_
{b}
-
w^{(t)}_{b}\|^2 \nonumber\\& + \sum_{(j,l)\in b}\mathbb{I}_
{\geq 0}(z^{(t)}_
{j,l}) +
\sum_{
(j,l)\notin b}\mathbb{I}_{\geq 0}(w^{(t)}_{j,l})\Big],\label{eq:use_smooth}
\end{align}
where we have used Lemma~\ref{lem:blocksmooth} to obtain 
\eqref{eq:use_smooth}, and $z^{(t)}_{b}$ is defined
as in \eqref{eq:rewrite_update1} for any block $b=(k, \mathcal{K})$.

We now need to aggregate the blocks over the sum in \eqref{eq:use_smooth},
taking into account the overlapping structure of our blocks.
We rely on the observation that each coordinate $w_{k,l}$
appears in exactly $2
{K-2\choose\kappa-1}$ blocks.
%
Grouping coordinates accordingly in \eqref{eq:use_smooth} gives:

\begin{equation}
\label{eq:H_appear}
\begin{aligned}
\mathbb{E}_{b_t}[F(w^{(t+1)})] &\leq \frac{K{K-1\choose \kappa}-\kappa{K-2\choose \kappa-1}}{K
{K-1\choose \kappa}}F(w^{(t)})\\ &+ \frac{\kappa{K-2\choose \kappa-1}}{K
{K-1\choose \kappa}}\Big( h(w^{(t)}) +
(\tilde{z}^{(t)} - w^{(t)})\nabla h(w^{(t)}) + \frac{L_{max}}{2}\|\tilde{z}^{
(t)} - w^{(t)}\|^2  + r(\tilde{z}^{(t)}) \Big),
\end{aligned}
\end{equation}
where $\tilde{z}^{(t)}$ is defined as in \eqref{eq:ztilde}. This is because the
block $b$ of $\tilde{z}^{(t)}$ is equal to $z^{(t)}_{b}$, as explained in
Section~\ref{sec:prox}.

We now deal with the second term in \eqref{eq:H_appear}.
Let us consider the following function $H$:
$$H(w^{(t)}, z) = h(w^{(t)}) + (z - w^{(t)})^\top\nabla h(w^{(t)}) + \frac{L_{max}}{2}
\|z - w^{(t)}\|^2 + r(z).$$
By $\sigma$-strong convexity of $h$, we have:
\begin{align}
H(w^{(t)}, z) &\leq h(z) - \frac{\sigma}{2}\|z - w^{(t)}\|^2 + \frac{L_{max}}{2}
\|z - w^{(t)}\|^2 + r(z)\nonumber\\
&= F(z) + \frac{1}{2}(L_{max}-\sigma)\|z - w^{(t)}\|^2.\label{eq:H2}
\end{align}

Note that $H$ achieves its minimum at $\tilde{z}^{(t)}$ defined in \eqref{eq:ztilde}.
If we minimize over $z$ both sides of \eqref{eq:H2} we get
\begin{align*}
H(w^{(t)}, \tilde{z}^{(t)}) &= \min_z H(w^{(t)}, z)\\
&\leq \min_z F(z) + \frac{1}{2}(L_{max}-\sigma)\|z - w^{(t)}\|^2.
\end{align*}

By $\sigma$-strong convexity of $F$,\footnote{$F=h+r$ is $\sigma$-strongly convex since $h$ is $\sigma$-strongly
convex and $r$ is convex.} we have for any $w,w'$ and $\alpha\in[0,1]$:
\begin{equation}
\label{eq:strongF}
F(\alpha w + 
(1-\alpha)w') \leq \alpha F(w) + (1-\alpha)F(w') - \frac
{\sigma\alpha(1-\alpha)}{2}\|w-w'\|^2.
\end{equation}

Using the change of variable $z=\alpha w^* + (1-\alpha)w^{(t)}$ for
$\alpha\in[0,1]$ and \eqref{eq:strongF} we
obtain:
\begin{align}
H(w^{(t)}, \tilde{z}^{(t)}) &\leq \min_{\alpha\in[0, 1]} F\big(\alpha w^* + 
(1-\alpha)w^{(t)}\big) + \frac{1}{2}(L_{max}-\sigma)\alpha^2\|w^* - w^{(t)}\|^2\nonumber\\
&\leq \min_{\alpha\in[0, 1]} \alpha F(w^*) + (1-\alpha)F(w^{(t)}) + \frac{1}
{2}\big[(L_{max}-\sigma)\alpha^2 - \sigma\alpha(1-\alpha)\big]\|w^* - w^{
(t)}\|^2\label{eq:preH3}\\
&\leq \frac{\sigma}{L_{max}}F(w^*) + \Big(1-\frac{\sigma}{L_{max}}\Big) F(w^
{(t)}),\label{eq:H3}
\end{align}
where the last inequality is obtained by plugging the value
$\alpha=\sigma/L_{max}$, which cancels the last term in \eqref{eq:preH3}.

We can now plug \eqref{eq:H3} into \eqref{eq:H_appear} and subtract $F(w^*)$
on both sides to get:
\begin{align}
\mathbb{E}_{b_t}[F(w^{(t+1)})] - F(w^*) &\leq \frac{K{K-1\choose \kappa}-2
{K-2\choose \kappa-1}}{K
{K-1\choose \kappa}}F(w^{(t)}) + \frac{2{K-2\choose \kappa-1}}{K
{K-1\choose \kappa}}\Big(\frac{\sigma}{L_{max}}F(w^*) + \Big(1-\frac{\sigma}{L_{max}}\Big) F(w^
{(t)})\Big)\nonumber\\
&= \Big(1- \frac{2\kappa\sigma}{K(K-1)L_{max}}\Big)(F(w^{(t)}) - F
(w^*))\label{eq:cdprooffinal},
\end{align}
where we used the fact that $\frac{2{K-2\choose \kappa-1}}{K{K-1\choose
\kappa}} =
\frac{2\kappa}{K(K-1)}$.
We conclude by taking the expectation of both sides with respect to the choice
of previous blocks $b_0, \dots, b_{t-1}$ followed by a recursive application
of the resulting formula. 

\section{SMOOTHNESS AND STRONG CONVEXITY OF GRAPH LEARNING FORMULATION}
\label{sec:smooth}

We derive the (block) smoothness and strong convexity parameters of the
objective function $h(w)$ when we use
\begin{equation}
\label{eq:kalo}
g(w) = \lambda\|w\|^2 - \mathbf{1}^\top\log(d(w) + \delta).
\end{equation} 

\textbf{Smoothness.} A function $f:\mathcal{X}\rightarrow\mathbb{R}$ is
$L$-smooth w.r.t. the Euclidean norm if its gradient is $L$-Lipschitz, i.e. $\forall(x,x')\in\mathcal{X}^2$:
$$\|\nabla f(x) - \nabla f(x')\| \leq L \|x - x'\|.$$

In our case, we need to analyze the smoothness of our objective function $h$
for each block of coordinates indexed by $(k, \mathcal{K})$. Therefore for any $(w,w')\in\mathbb{R}_+^{K(K-1)/2}$ which differ only in the $(k, \mathcal{K})$-block, we want to find $L_{k,\mathcal{K}}$ such that:
$$\|[\nabla h(w)]_{k,\mathcal{K}} - [\nabla h(w')]_{k,\mathcal{K}}\| \leq L_
{k,\mathcal{K}} \|w_{k,\mathcal{K}} - w'_{k,\mathcal{K}}\|.$$

\begin{lemma}
For any block $(k, \mathcal{K})$ of size $\kappa$, we have
$L_{k,\mathcal{K}} \leq
\mu_1(\frac{\kappa+1}{\delta^2} + 2\lambda)$.
\end{lemma}
\begin{proof}
Recall from the main text that the partial gradient can be written as follows:
\begin{equation}
\label{eq:newgrad}
[\nabla h(w)]_{k, \mathcal{K}} = p_{k, \mathcal{K}} + (\mu_1/2)\Delta_{k, 
\mathcal{K}} + \mu_2 v_{k, \mathcal{K}}(w),
\end{equation}
where 
$p_{k, \mathcal{K}} = (c_k\mathcal{L}_k(\alpha_k; S_k) + c_l
\mathcal{L}_l(\alpha_l; S_l))_{l\in\mathcal{K}}$,
$\Delta_{k, \mathcal{K}} = 
(\|\alpha_k-\alpha_l\|^2)_{l\in\mathcal{K}}$ and
$v_{k, \mathcal{K}}(w) = (
g_k'(d_k(w)) +g_l'(d_l(w)) + g'_{k,l}(w_{k,l}))_{l\in\mathcal{K}}$.
In our case, $g(w)$ is defined as in \eqref{eq:kalo} so we have $v_{k, \mathcal{K}} =
(\frac{1}{d_k(w) + \delta} + \frac{1}{d_l(w) + \delta} + 2\lambda w_{k, l})_{l\in
\mathcal{K}}$. Note also that we set $\mu_2=\mu_1$, as discussed in the main
text.

The first two terms do not depend on $w_{k, \mathcal{K}}$ and can thus be ignored. We focus on the Lipschitz constant corresponding to the third term.
Let $(w,w')\in\mathbb{R}_+^{K(K-1)/2}$ such that they only differ in the
block indexed by $
(k,\mathcal{K})$. We denote the degree of an user $k$ with respect to $w$ and
$w'$ by $d_k(w) = \sum_{j} w_{k, j}$ and $d'_k(w) = \sum_{j} w'_{k, j}$
respectively. Denoting $z_{k, \mathcal{K}} = (\frac{1}{d_k(w) + \delta} + 
\frac{1}{d_l(w) + \delta})_{l\in
\mathcal{K}}$, we have:

\begin{align}
\| z_{k, \mathcal{K}} - z'_{k, \mathcal{K}}\| &= \Big\| \Big( \frac{1}
{d_k(w)+\delta} + \frac{1}{d_l(w) + \delta} \Big)_{l\in\mathcal{K}} - \Big(
\frac{1}
{d'_k(w)+\delta} + \frac{1}{d'_l(w) + \delta} \Big)_{l\in\mathcal{K}}
\Big\|\nonumber\\
&= \Big\| \Big( \frac{d_k'(w)+\delta - d_k(w) - \delta}
{(d_k(w)+\delta)(d'_k(w)+\delta)} + \frac{d'_l(w) + \delta - d_l(w) - \delta}{(d_l(w) +
\delta)(d'_l(w) + \delta)} \Big)_{l\in
\mathcal{K}} \Big\|\nonumber\\
&\leq \frac{1}{\delta^2} \Big\| \Big(\sum_{j\in\mathcal{K}}(w'_{k, j} - w_{k, j}
)\Big)_{l\in\mathcal{K}} + \Big( w'_{k, l} - w_{k, l} \Big)_{l\in
\mathcal{K}} \Big\|\label{eq:lip_pos}\\
&\leq \frac{1}{\delta^2}\Big [ \Big\| \Big(\Big|\|w'_{k, \mathcal{K}}\|_1 -
\|w_{k,
\mathcal{K}}\|_1 \Big|\Big)_{l\in
\mathcal{K}} \Big\| + \|w_{k, \mathcal{K}} - w'_{k, 
\mathcal{K}}\|\Big]\nonumber\\
&\leq \frac{1}{\delta^2}\Big [ \Big\| \Big(\|w'_{k, \mathcal{K}}-w_{k,
\mathcal{K}}\|_1 \Big)_{l\in
\mathcal{K}} \Big\| + \|w_{k, \mathcal{K}} - w'_{k, 
\mathcal{K}}\|\Big]\label{eq:lip_norm1}\\
&\leq \frac{1}{\delta^2}\Big [ \sqrt{\kappa}\Big| \sum_{j\in\mathcal{K}}(w'_
{k,j} - w_{k,j}) \Big| + \|w_{k, \mathcal{K}} - w'_{k,
\mathcal{K}}\|\Big] \label{eq:lip_norm2}\\
&\leq \frac{1}{\delta^2}\Big [ \sqrt{\kappa} \|w_{k, \mathcal{K}} - w'_{k,
\mathcal{K}}\|_1 + \|w_{k, \mathcal{K}} - w'_{k,
\mathcal{K}}\|\Big] \label{eq:lip_norm3}\\
&\leq \frac{\kappa+1}{\delta^2}\|w_{k, \mathcal{K}} - w'_{k, 
\mathcal{K}}\|,\nonumber
\end{align}
where to obtain \eqref{eq:lip_pos} we used the nonnegativity of the weights
and the fact that $w$ and $w'$ only differ in the coordinates indexed by $
(k,\mathcal{K})$, and \eqref{eq:lip_norm1}-\eqref{eq:lip_norm2}-\eqref{eq:lip_norm3} by classic
properties of norms.



We conclude by combining this result with the
quantity $2\lambda$ that comes from the last term in $v_{k, \mathcal{K}}$ and
multiplying by $\mu_1$.
\end{proof}
Observe that $L_{k,\mathcal{K}}$ only depends on the block size $\kappa$ (not
the block itself). Hence we also have $L_{max} \leq \mu_1(\frac{\kappa+1}
{\delta^2} + 2\lambda)$.

It is important to note that the linear dependency of $L_{k,\mathcal{K}}$ in the block
size $\kappa$, which is the worst possible dependency for block
Lipschitz constants \citep{wright2015coordinate}, is \emph{tight} for our objective function.
This is due to the log-degree term which makes each entry of $[\nabla h(w)]_
{k, \mathcal{K}}$ dependent on the sum of all coordinates in $w_{k, 
\mathcal{K}}$. This linear dependency explains the mild effect of
the block size $\kappa$ on the convergence rate of our algorithm (see the
discussion of Section~\ref{sec:disco-com} and the numerical results of
Appendix~\ref{app:kappa}).

\textbf{Strong convexity.} It is easy to see that the objective function $h$
is
$\sigma$-strongly convex with $\sigma=2\mu_1\lambda$.


%% file: supp_comm_mem.tex

\section{COMMUNICATION AND MEMORY}
\label{sec:communcation-memory}

In this section we provide additional details on the communication and memory
costs of the proposed method. The section is organized in two parts
corresponding to  the decentralized boosting algorithm of Section~\ref{sec:method} and the graph learning algorithm of Section~\ref{sec:discovery}. 

\subsection{Learning Models: A Logarithmic Communication and Memory Cost}
\label{sec:com}
We prove that our Frank-Wolfe algorithm of Section~\ref{sec:method} enjoys logarithmic
communication and memory costs with respect to the number of base predictors $n$.
Combined with the approach for building sparse collaboration graphs we introduce in
Section~\ref{sec:discovery}, we obtain a scalable-by-design algorithm.
The following analysis stands for systems without failure (all sent messages
are correctly received). We express all costs in number of bits, and $Z$
denotes the bit length used to represent floats.  Assume we are given a 
collaboration graph $ \mathcal{G} = ([K], E, w)$ with $K$ nodes and $M$ edges.

Recall that the algorithm proceeds as follows. At each time step $t$, a
random user $k$ wakes up and performs the following actions:
\begin{enumerate}
\item \emph{Update step:} user $k$ performs a Frank-Wolfe update on its local
model based on the most recent information $\alpha_l^{\mtime{t-1}}$ received from
its neighbors $l\in N_k$:
$$\alpha_k^{\mtime{t}} = (1 - \kt{\gamma}{}{t}) \kt{\alpha}{k}{t-1} + 
\kt{\gamma}{}{t} \: \kt{s}{k}{t},\quad\text{with }\kt{\gamma}{}{t} = 2K/(t
+2K).$$
\item \emph{Communication step:} user $k$ sends its updated model $\alpha_k^{\mtime{t}}$ to its neighborhood $N_k$.
\end{enumerate}

\textbf{Memory.}
Each user needs to store its current model, a copy of its neighbors' models,
and the similarity weights associated with its neighbors. Denoting by $|N_k|$
the number of neighbors of user $k$, its
memory cost is given by $Z(n + |N_k|(n+1))$,
which leads to a total cost for the network of
$$ K Z\big( n + \textstyle\sum_{k=1}^K|N_k|(n + 1)\big) =  Z \left( K n + 2 M 
(n + 1)\right). $$
The total memory is thus linear in $M$, $K$ and $n$.
Thanks to the sparsity of the updates, the dependency on $n$ can be reduced from linear to logarithmic by representing models as sparse vectors. Specifically, when initializing the models to zero vectors, the model of an user $k$ who has performed $t_k$ updates so far contains at most $t_k$ nonzero elements and can be represented using $t_k (Z+\log n)$ bits: $t_k Z $ for the nonzero values and $t_k \log n$ for their indices.

\textbf{Communication.}
At each iteration, an user $k$ updates a single coordinate of its model
$\alpha_k$. Hence, it is enough to send to the neighbors the index of the
modified coordinate and its new value (or the index and the step
size $\gamma_k^{(t)}$). Therefore, the communication cost of a single
iteration is equal to $ (Z+\log n) |N_k|$. 
After $T$ iterations, the expected total communication cost for our
approach is
$$ T (Z+\log n) \expect{|N_k|} = 2 T M K^{-1}(Z+\log n).$$
Combining this with Theorem~\ref{the:optimal}, the total communication cost
needed to obtain an optimization error smaller than $\varepsilon$ amounts to
 $\frac{12 M (C^\otimes_f + h_0)}{\varepsilon}\big(Z+\log n\big)$, hence
 logarithmic in $n$. For the classic case where the set of base predictors
 consists of a  constant number of simple decisions stumps per feature, this
 translates into  a logarithmic cost in the \emph{dimensionality of the data} 
 (see the  experiments of Section~ \ref{sec:exp}). This can be much smaller
 than the  cost needed to send all the data to a central server.

\subsection{Learning the Collaboration Graph: Communication vs. Convergence}
\label{sec:disco-com}

Recall that the algorithm of Section~\ref{sec:discovery} learns in a fully decentralized way a collaboration graph given fixed models $\alpha$. It is defined by:

\begin{enumerate}
\item Draw a set $\mathcal{K}$ of $\kappa$ users and request their current models and degree.
\item Update the associated weights:
$ w^{\mtime{t+1}}_{k, \mathcal{K}} \leftarrow \max\big(0, w^{\mtime{t}}_{k, \mathcal{K}} - 
(1/L_{k,\mathcal{K}})[\nabla h(w^{\mtime{t}})]_{k, \mathcal{K}}\big),
$
\item Send each updated weight $w^{\mtime{t+1}}_{k,l}$ to the associated user in $l\in\mathcal{K}$.
\end{enumerate}



At each iteration, the active user needs to request from each user $l\in
\mathcal{K}$ its current degree $d_l(w)$, its personal model $\alpha_l^
{
(t_\alpha)}$ and the value of its local loss, where
$t_\alpha$
is the total number of FW model updates done so
far in the network. It then sends the updated
weight to each
user in $\mathcal{K}$. As the expected number of nonzero entries in the model
of an user is at most $\min(t_\alpha/K, n)$, the expected
communication cost for a
single iteration is equal to $\kappa (3Z + \min(\frac{t_\alpha}{K}, n)(Z+\log
n))$, where $Z$ is the representation length of a float.
This can be further
optimized if users have enough local memory to store the models and local
losses of all the users they communicate with (see Section~
\ref{sec:refined_com} below).

In general, Theorem~\ref{the:disco} shows that the parameter $\kappa$ can be
used to trade-off the convergence speed and the amount of communication needed
at each iteration, especially when the number of users $K$ is large.
For the particular case of $g(w) = \lambda\|w\|^2 - 
\mathbf{1}^\top\log(d(w) + \delta)$ that we propose, we have
$\sigma=2\mu\lambda$ and $L_{max} \leq
\mu(\frac{\kappa+1}{\delta^2} + 2\lambda)$ (see
Section~\ref{sec:smooth}). This gives the following shrinking factor in
the convergence rate of Theorem~\ref{the:disco}:
$$\rho \leq 1 - \frac{4}{K(K-1)}\frac{\kappa\lambda\delta^2}
{\kappa+1+2\lambda\delta^2}.
$$
Hence, while increasing $\kappa$ results in a linear increase in the
per-iteration communication cost (as well as in the number of users to
communicate with), the impact on $\rho$ in Theorem~\ref{the:disco} is mild and
fades rather quickly due to the (tight) linear
dependence of $L_{k,\mathcal{K}}$ in $\kappa$. This suggests that
choosing
$\kappa=1$ will minimize the
total communication cost needed to reach solutions of moderate precision 
(which is usually sufficient for machine learning).
Slightly larger values (but still much smaller than $K$) will provide a better
balance between the communication cost and the number of rounds. On the
other hand, if high precision solutions are needed or if the number of
communication rounds is the primary concern, large values of $\kappa$ could
be used. As shown in Section~\ref{app:kappa}, the numerical behavior of
our algorithm is in line with this theoretical analysis.

\subsubsection{Refined communication complexity analysis}
\label{sec:refined_com}

The communication
complexity of our
decentralized graph learning algorithm can be reduced if the users store
the models and local losses of all the peers they communicate with.
The communication complexity for a given iteration $t$ then depends on the
expected number of nodes $\bar \kappa^{(t)}$, among the selected $\mathcal{K}$,
that the picked user has not yet selected:
$$\kappa 2 Z + \bar \kappa^{(t)} \Big(Z + \min\big(\frac{t_\alpha}{K}, n\big)
(Z+\log n))
\Big).$$
The next lemma shows that $\bar \kappa^{(t)}$ decreases exponentially fast with
the number of iterations.

\begin{proposition}
For any $T\geq 1$, the expected number of new nodes after $T$ iterations is
given by
$$\bar \kappa^{(T)} = \kappa \left( 1- \frac{\kappa}{K(K-1)} \right)^
{T-1}.$$
\end{proposition}

\begin{proof}
    
At a given iteration $t$, let $k^{(t)}$ denote the random user
that performs the update
and $\mathcal{K}^{(t)}$ the set of $\kappa$ users selected by $k^{(t)}$. We
denote by $X_{l,m}^{
(t)}$ the random variable indicating if node $l$ selected node $m$ at that
iteration:
\[   
X_{l,m}^{(t)} = 
     \begin{cases}
       0& \text{if node} \ m \ \text{was not selected by node} \ l \ \text{at iteration} \ t, \\
       1& \text{otherwise.} 
     \end{cases}
\]
Similarly, $X_{l,m}$ indicates if node $l$ has ever selected node $m$ after
$T$ iterations.

Let us denote by $R^{(t)} = \{X_{k^{(t)},j}\}_{j \in \mathcal{K}^{(t)}}$ the
set
of random variables that have to be updated at iteration $t$.
The probability that node $m$ is not selected by node $l$ at a given round $t$ is given by
\begin{align*}
    \mathbb{P}[X_{l,m}^{(t)} = 0] &= \mathbb{P}[k^{(t)} = l] \mathbb{P}[X_{k^{(t)},m} \notin R^{(t)} | k^{(t)} = l] + \mathbb{P}[k^{(t)} \neq l] \mathbb{P}[X_{k^{(t)},m} \notin R^{(t)} | k^{(t)} \neq l] \\
    &= \frac 1 K \left(1 - \frac{\kappa}{K-1} \right) + \frac{K-1}{K} 1 \\
    &= \frac{K^2 - K - \kappa}{K (K - 1)} = 1- \frac{\kappa}{K(K-1)}.
\end{align*}

As $k$ and $\mathcal{K}$ are drawn independently from the previous draws, the
probability that node $m$ has never been selected by node $l$ after $T$
iterations is given by:
$$ \mathbb{P}[X_{l,m} = 0] = \prod_{t=0}^{T-1} \mathbb{P}[X_{l,m}^{(t)} = 0] =
\left(1- \frac{\kappa}{K(K-1)} \right)^{T}.$$

Finally, the expected number of new nodes seen at iteration $T$ is given by
\begin{align*}
  \bar \kappa^{(T)} &= \mathbb{E}_{k,\mathcal{K}} [ \ |\{ X_{k,l} \in R^{(T)} | X_{k,l} = 0 \}| \ ] \\
               &= \frac 1 K \sum_{k=1}^{K} \frac{1}{{{K-1}\choose{\kappa}}} \sum_{\mathcal{K}} \sum_{m \in \mathcal{K}} \mathbb{P}[X_{l,m} = 0] \\
               &= \kappa \left(1- \frac{\kappa}{K(K-1)} \right)^{T-1}.
\end{align*}
\end{proof}

%% file: supp_exp.tex
\section{ADDITIONAL EXPERIMENTS}
\label{app:exp}


\subsection{Details on Experimental Setting}

\paragraph{Hyperparameter tuning.}
We tune the following hyper-parameters with 3-fold cross validation on the
training user datasets: $\beta\in\{ 1, \dots, 10^3 \}$ ($l_1$ constraint for
all Adaboost-based methods), $\mu\in\{ 10^{-3}, \dots, 10^3
\}$ 
(trade-off parameter for \algoname and \persolin), and
$\lambda\in\{10^{-3},
\dots, 10^3 \}$ (graph sparsity in \algonameL and \persolinL).

\paragraph{Description of Moons dataset.}

We describe here in more details the generation of the synthetic problem {\sc
Moons} used in the main text, which is constructed from the classic two
interleaving Moons dataset which has nonlinear class boundaries.
We consider $K=100$ users, clustered in $4$ groups of respectively $K_
{c_1}=10$, $K_{c_2}=20$, $K_{c_3}=30$ and $K_{c_4}=40$ users.
Each cluster is associated with a rotation angle $\Theta_{c}$ of
$45$, $135$, $225$ and $315$ degrees respectively.
We generate a local dataset for each user $k$ by drawing $m_k\sim\mathcal{U}
(3,15)$ training examples and $100$ test examples from the
two Moons distribution. We then apply a rotation
(coplanar to the Moons' distribution) to all the points according to an angle
$\theta_k\sim\mathcal{N}(\Theta_c, 5)$
where $c$ is the cluster the user belongs to.
This construction allows us to control the similarity between
users (users from the same cluster are more similar to each other
than to those from different clusters). We build an oracle
collaboration graph by
setting
$w_{k,l} = \exp ( \frac{\cos(\theta_{k} - \theta_{l}) - 1}{\sigma})$ with
$\sigma = 0.1$ and dropping all edges with negligible weights, which we will
give as input to \algonameF and \persolinF.
In order to make the classification problems more challenging, we
add random label noise to the generated local samples by flipping the labels
of $5\%$ of the training data, and embed all points in $
\mathbb{R}^D$ space by adding random values for the $D-2$ empty axes, similar
to \citep{vanhaesebrouck2016decentralized}.
In the experiments, we set $D=20$.

\paragraph{Description of the real datasets.} We give details on datasets used
in the main text:
\begin{itemize}
    \item {\sc Harws} (Human Activity Recognition With Smartphones)~
    \citep{anguita2013public}, which is
composed of records of various types of physical activities, described by $D=561$ features and collected from $K=30$ users. We focus on the task of distinguishing when a user is sitting or not, use $20\%$ of the records for training and set the number of stumps for the boosting-based methods to $n=1122$.

    \item {\sc Vehicle Sensor}~\citep{duarte2004vehicle} contains data from
    $K=23$ sensors describing vehicles driving on a road, where each record is described by $D=100$ features. We predict between AAV and DW vehicles, using $20\%$ of the records for training, and fix the number of stumps to $n=1000$.

    \item {\sc Computer Buyers}\footnote{{\scriptsize
    \url{https://github.com/probml/pmtkdata/tree/master/conjointAnalysisComputerBuyers}}} consists of $K=190$ buyers, who have each evaluated $m_k=20$ computers described by $D=14$ attributes, with an overall score within the range $[0, 10]$. We use a total of $1407$ (between $5$ and $10$ per user) instances for training and $2393$ (between $10$ and $15$ per user) for testing. We tackle the problem as binary classification, by affecting all instances with a score above $5$ to the positive class and the remaining ones to the negative class, and we set the number of stumps to $n=28$.

    \item {\sc School}~\citep{goldstein1991multilevel}\footnote{{\scriptsize
\url{https://github.com/tjanez/PyMTL/tree/master/data/school}}} consists of
$m=15362$ total student examination records described by $D=17$ features, with an overall score in the
range $[0, 70]$ from $K=140$ secondary schools. In
total, there are $11471$ instances (between $16$ and $188$ per user) for
training and $3889$ (between $5$ and $63$ per user) for
testing. We predict between records with scores smaller or greater than $20$ and set the number of stumps to $n=34$.
\end{itemize} 

\subsection[Effect of the Block Size]{Effect of the Block Size $\kappa$}
\label{app:kappa}

\begin{figure}[t]
    \centering
    \includegraphics[width=.7\columnwidth]{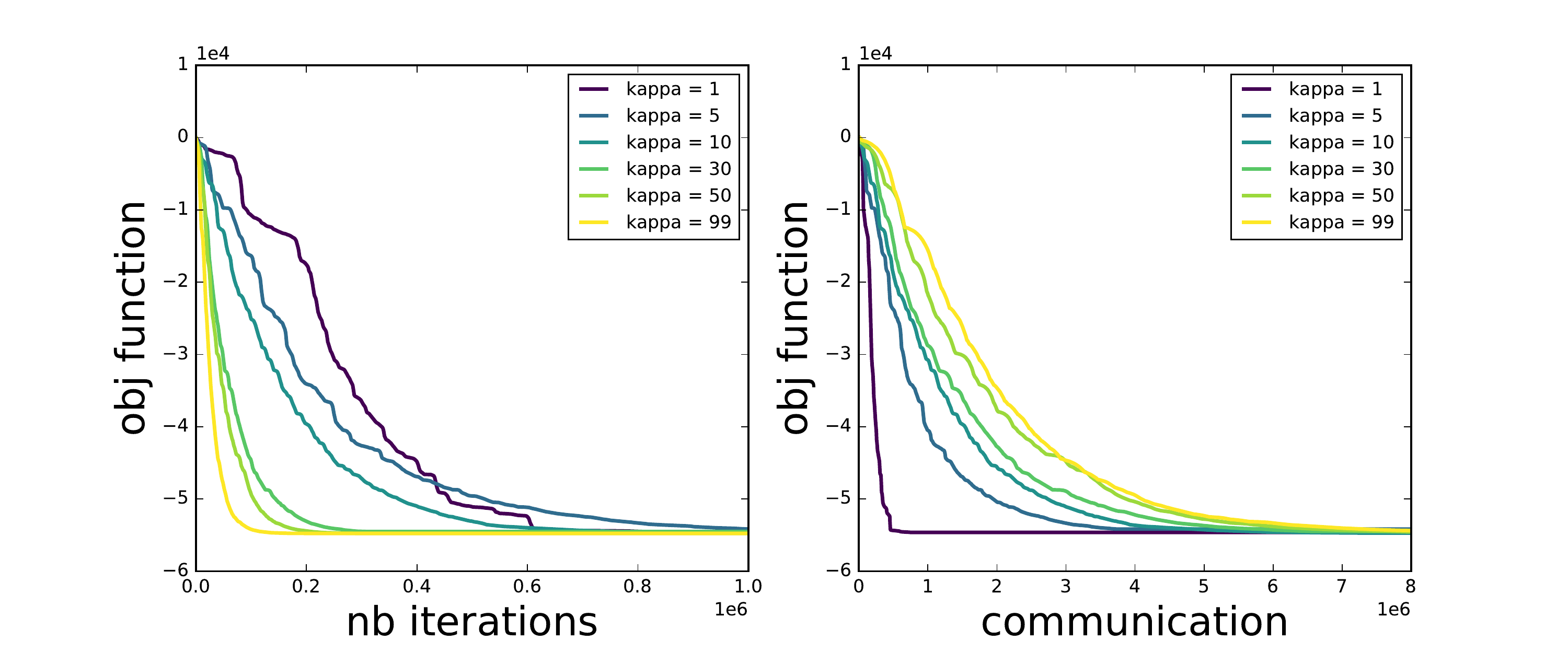}
    \caption{Impact of $\kappa$ on the
    convergence rate and the communication
    cost for learning the graph on {\sc Moons}.
    }
  \label{fig:kappas}
\end{figure}

As discussed in Section~\ref{sec:disco-com}, the parameter $\kappa$ allows to
trade-off the communication cost (in bits as well as the number of pairwise
connections at each iteration) and the convergence rate for
the graph learning steps of \algonameL.
We study the effect of varying $\kappa$ on our synthetic dataset {\sc Moons}.
Figure~\ref{fig:kappas} shows the evolution of the objective function with the number of iterations and with the communication cost depending on $\kappa$ when learning a graph using the local classifiers learned with \localboost.
Notice that the numerical behavior is consistent with the theory: while
increasing $\kappa$ reduces the number of communication rounds,
setting $\kappa=1$
minimizes the total amount of
communication (about $5\times 10^5$ bits). By way of comparison, the
communication cost
required to send all weights to all users just once is $Z K^2 (K - 1) / 2 =
1.6
\times 10^7$ bits. In practice, moderate values of $\kappa$ can
be used to obtain a good trade-off between the number of rounds and the
total communication cost, and to reduce the higher variance
associated
with small values of $\kappa$. 

\subsection{Test Accuracy with respect to Local Dataset Size}

In the main text, the reported accuracies are averaged over users. Here, we
study the relation between the local test accuracy of users
depending on the size of their training set. Figure~\ref{fig:points} shows a
comparison between \algonameF, \algonameL and \localboost, in
order to assess the improvements introduced by our collaborative scheme.
On {\sc Moons}, \localboost shows good
performance on users with larger training sets but generalizes poorly
on users with limited local information. Both \algonameF and \algonameL
outperform \localboost, especially on users with small datasets.
Remarkably, in the ideal setting where we have access to the ground-truth
graph (\algonameF), we are able to fully
close the accuracy gaps caused by uneven training set sizes.
\algonameL is able to match
this performance except on users with smaller datasets, which is expected
since there is very limited information available to learn reliable similarity
weights
for these users.
On {\sc Harws}, \algonameL generally improves upon \localboost, although
there is more variability due to difference in difficulty across user tasks
and uneven numbers of users in each size group.

\begin{figure}[t!]
    \centering
    \subfigure[{\sc Moons}.]{
    \label{fig:moons-points}
    \includegraphics[width=0.4\textwidth]{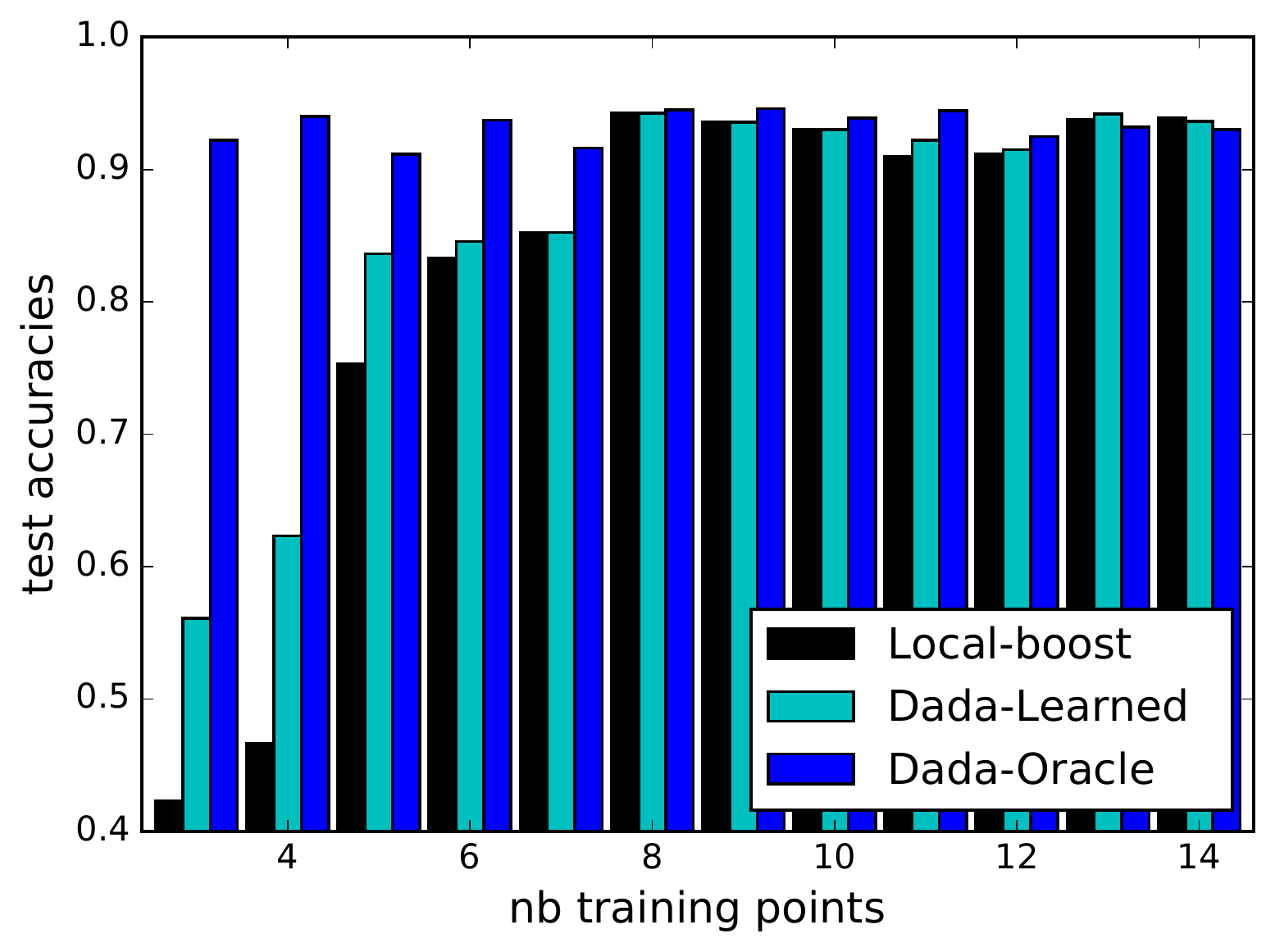}}
    \hspace{0.3cm}
    \subfigure[{\sc Harws}.]{
    \label{fig:harws-points}
    \includegraphics[width=.4\textwidth]{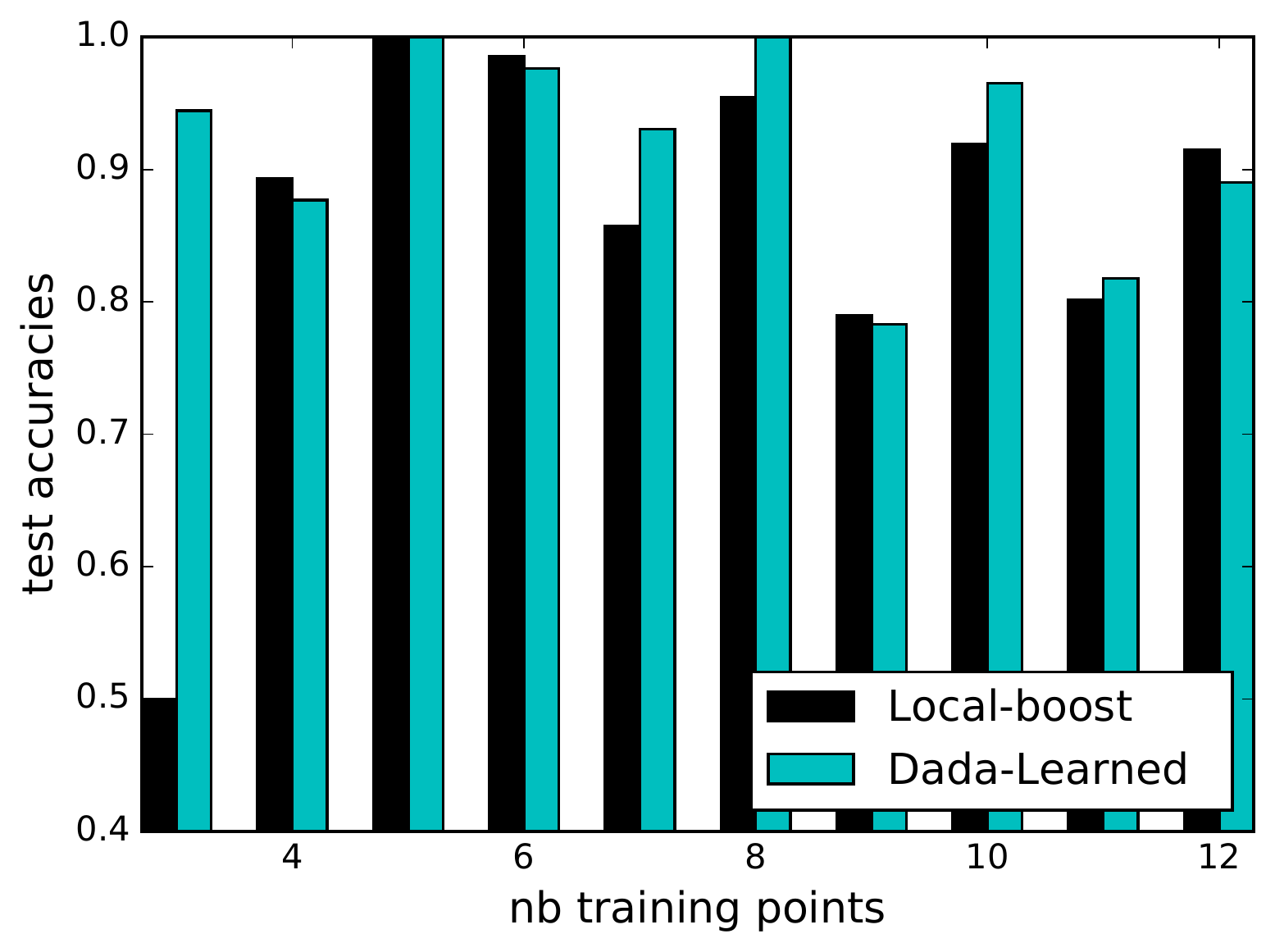}} 
    \caption{Average test accuracies with respect to the number of training points of the local sets.}
    \label{fig:points}
\end{figure}

\subsection[Test Accuracy with respect to Communication Cost]{Test Accuracy with respect to Communication Cost}
\label{app:comm}

We report the full study of the test accuracies under limited communication
budget, summarized in Table~\ref{tab:bud} of the main text.
Figure~\ref{fig:acc-com} confirms that \algonameL generally allows for reaching
higher test accuracies with less communications than \persolinL, especially on
higher-dimensional datasets, such as {\sc Harws} (Figure~\ref{fig:harws-comm}).

\begin{figure}[t!]
    \centering
    \subfigure[{\sc Harws} ($D=561$).]{
    \label{fig:harws-comm}
    \includegraphics[width=.42\textwidth]{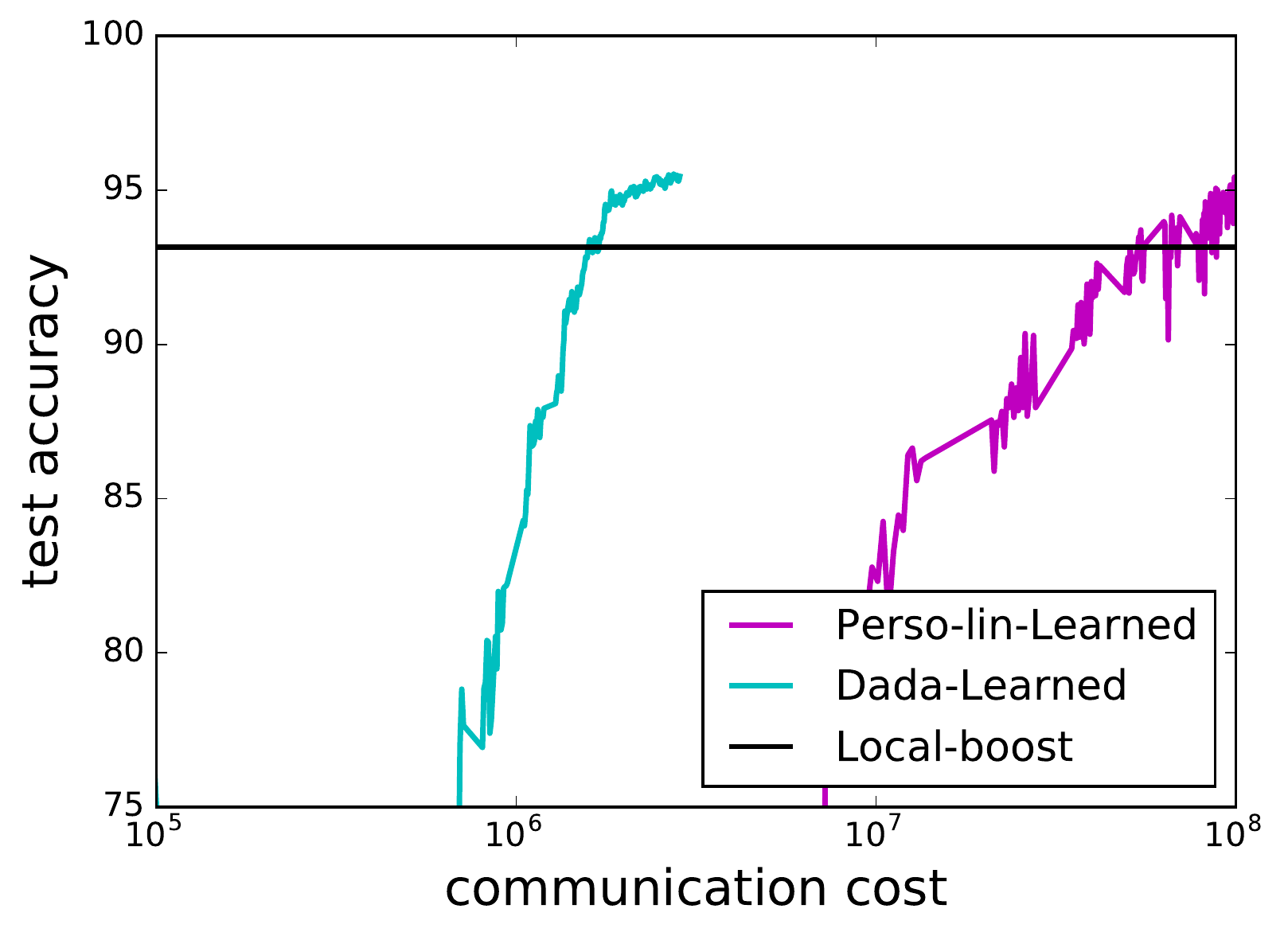}}
    \hspace{0.3cm}
    \subfigure[{\sc Vehicle Sensor} ($D=100$).]{
    \label{fig:vehicle-comm}
    \includegraphics[width=.4\textwidth]{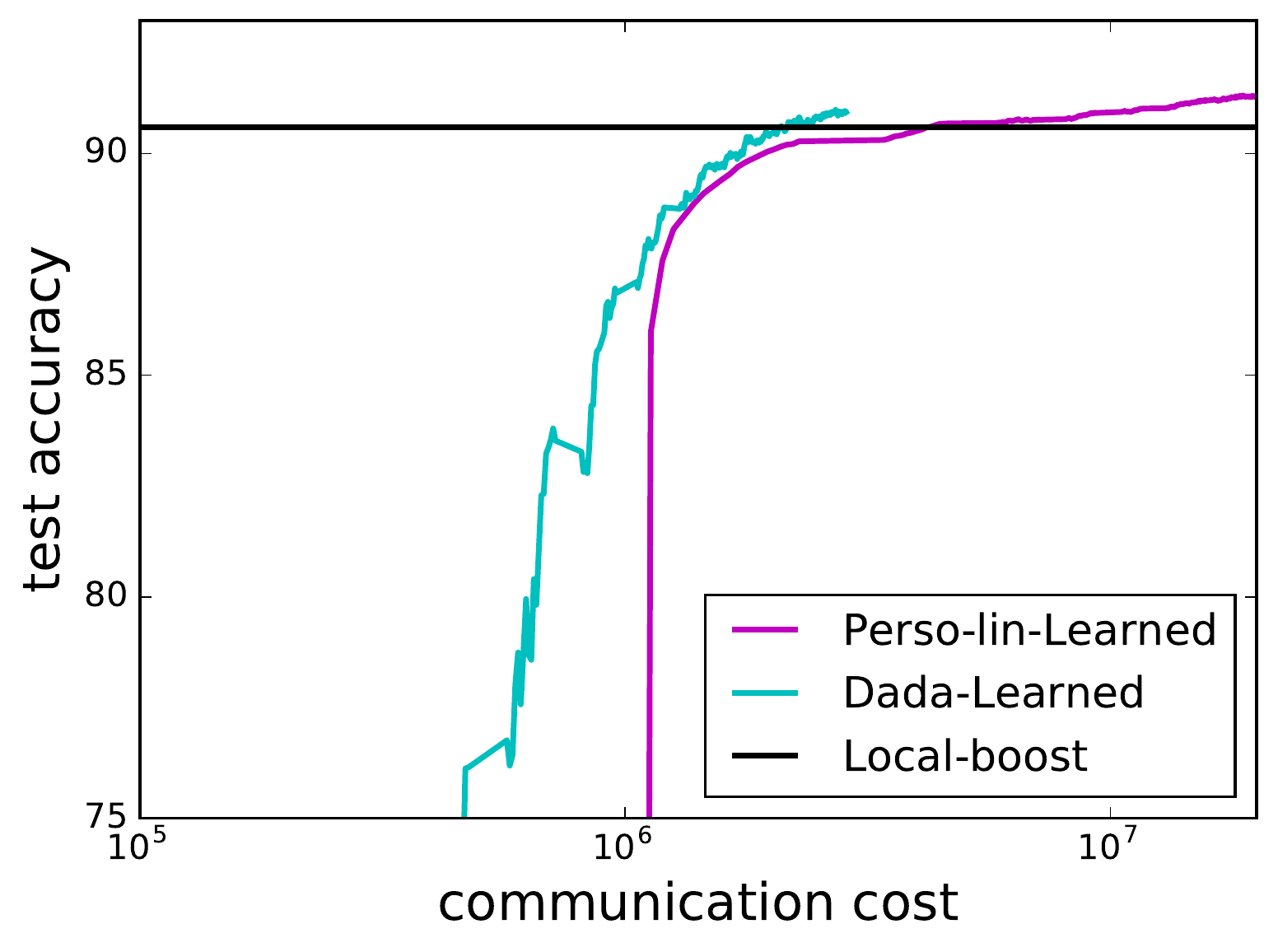}}
    \subfigure[{\sc Computer Buyers} ($D=14$).]{
    \label{fig:computer-comm}
    \hspace{0.3cm}
    \includegraphics[width=.4\textwidth]{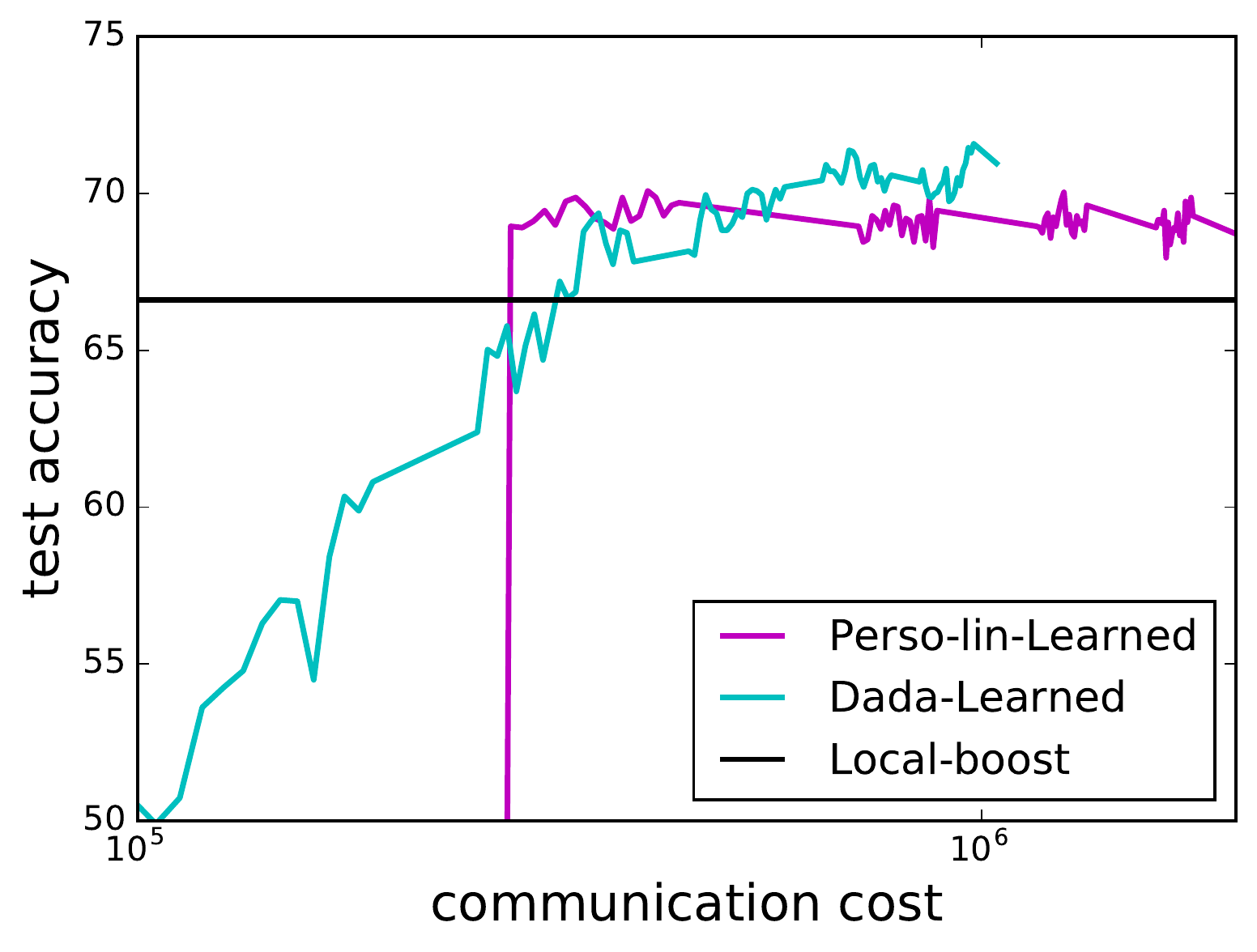}}
    \hspace{0.3cm}
    \subfigure[{\sc Schools} ($D=17$).]{
    \label{fig:school-comm}
    \includegraphics[width=.42\textwidth]{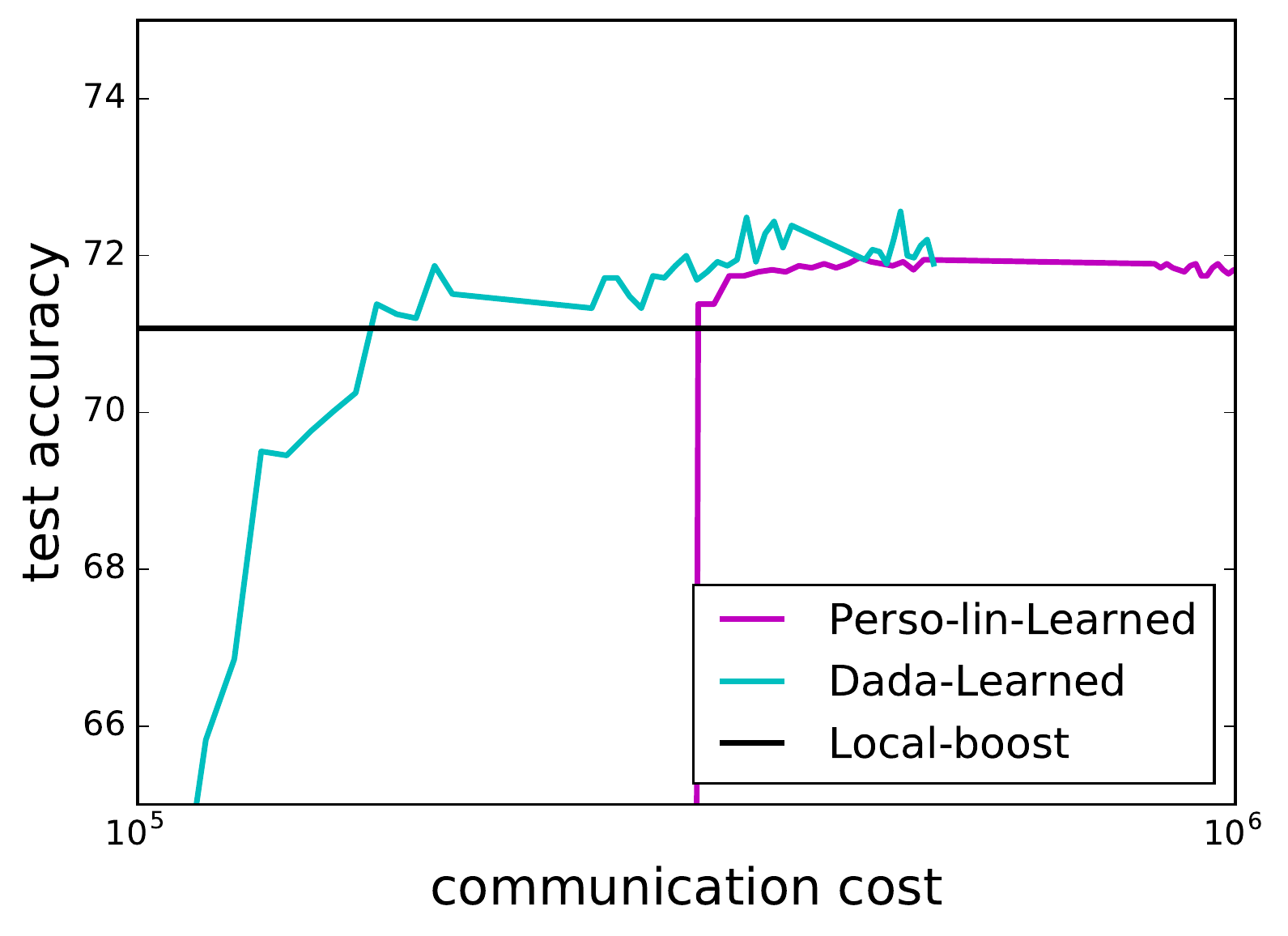}}

    \caption{Average test accuracies with respect to the communication cost (\# bits).}
    \label{fig:acc-com}
\end{figure}

\subsection{Standard Deviations due to Random Sampling}
We report the means and standard deviations of the test accuracies across 3
random runs to show how the randomness of the user selection in the
decentralized algorithms (Perso-linear-Learned and Dada-Learned) affects the results.
Notice that, because global and local methods are deterministic, their
standard deviations are zero.

\begin{table*}
  \caption{Test accuracy (\%) on real datasets. Best results in boldface and second best in italic.}
  \label{tab:acc}
  \centering
  ~\\\vspace{1mm}
  \scalebox{0.85}{
  \begin{tabular}{lcccc}
    \bf{DATASET}      & {\sc \bf{HARWS}}  & {\sc \bf{VEHICLE}} & {\sc \bf{COMPUTER}} & {\sc \bf{SCHOOL}} \\
    \midrule
    Global-linear & 93.64 & 87.11 & 62.18 & 57.06 \\
    Local-linear & 92.69 & 90.38 & 60.68 & 70.43\\
    Perso-linear-Learned  & \first{96.87 $\pm$ 0.97} & \first{91.45 $\pm$ 0.16} & 69.10 $\pm$ 0.05 & \second{71.78 $\pm$ 0.42}\\
    Global-Adaboost & 94.34 & 88.02 & \second{69.16 }& 69.96 \\
    Local-Adaboost     & 93.16 & 90.59 & 66.61 & 70.69 \\
    Dada-Learned & \second{95.57 $\pm$ 0.21} & \second{91.04 $\pm$ 0.70} & \first{73.55 $\pm$ 0.28} & \first{72.47 $\pm$ 0.81}\\
  \end{tabular}}
\end{table*}

\subsection{Additional Synthetic Dataset: Moons100}

We report the experiments carried out on a synthetic dataset referred
to as {\sc Moons100}, which is also based on the two interleaving Moons
dataset but with a different ground-truth task similarity structure.
We consider a set of $K=100$ users, each associated with a personal rotation
axis drawn from a normal distribution.
We generate the local datasets by drawing a random number of points from
the two Moons distribution: uniformly between $3$ and $20$ for training and
$100$ for testing. We then apply the random rotation of the user to all its
points.
We further add random label noise by flipping the labels of $5\%$ of the
training data and embed all the points in $\mathbb{R}^D$ space by
adding random values for the $D-2$ empty axes.
In the experiments, the number of dimensions $D$ is fixed to $20$ and the
number of base functions $n$ to $200$. For \algonameL, the graph
is updated after every $200$ iterations of optimizing $\alpha$.
We build an oracle collaboration graph where the weights between
users are computed from the angle $\theta_{ij}$ between the users' rotation axes, using $w_{i,j} = \exp ( \frac{\cos(\theta_{ij}) - 1}{\sigma} )$ with $\sigma = 0.1$.
We drop all edges with negligible weights. 


\begin{figure}
    \centering \includegraphics[width=0.6\textwidth]{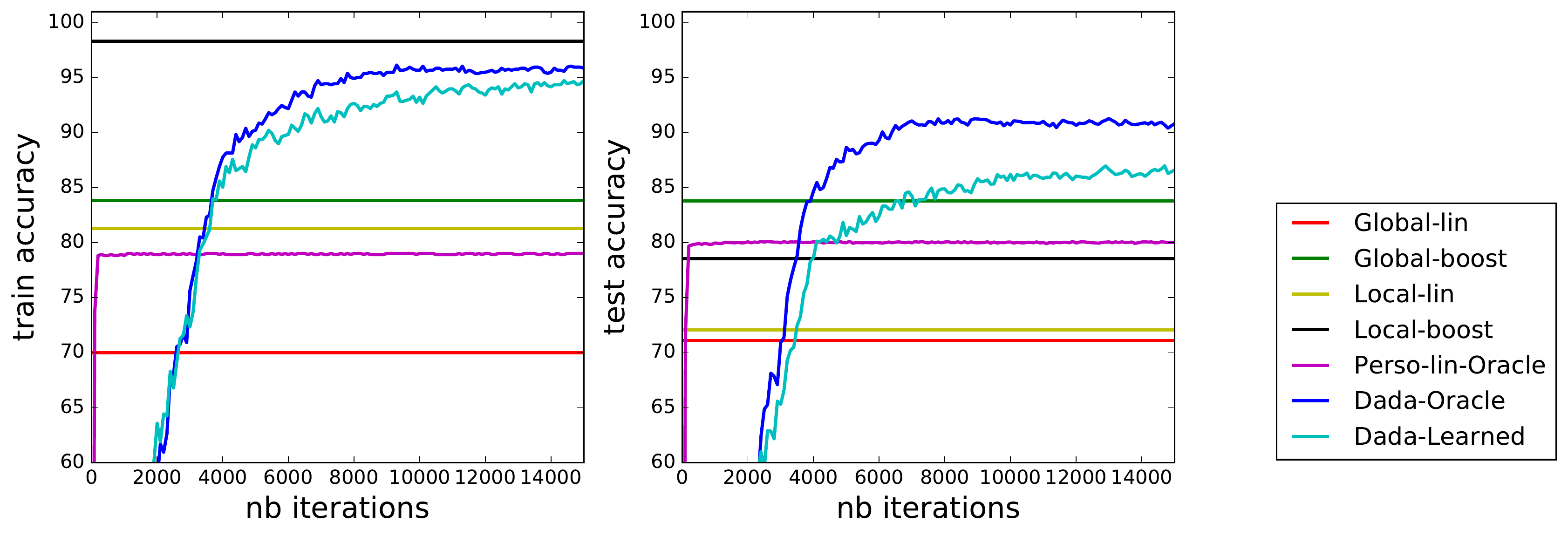}
    \caption{Training and test accuracy \wrt number of iterations on {\sc
    Moons100}.}
  \label{fig:100accuracies}
\end{figure}

\begin{figure}[t!]
    \centering
    \subfigure[Illustration of the learned and oracle graphs.]{
    \label{fig:visu}\includegraphics[width=0.44\textwidth]
    {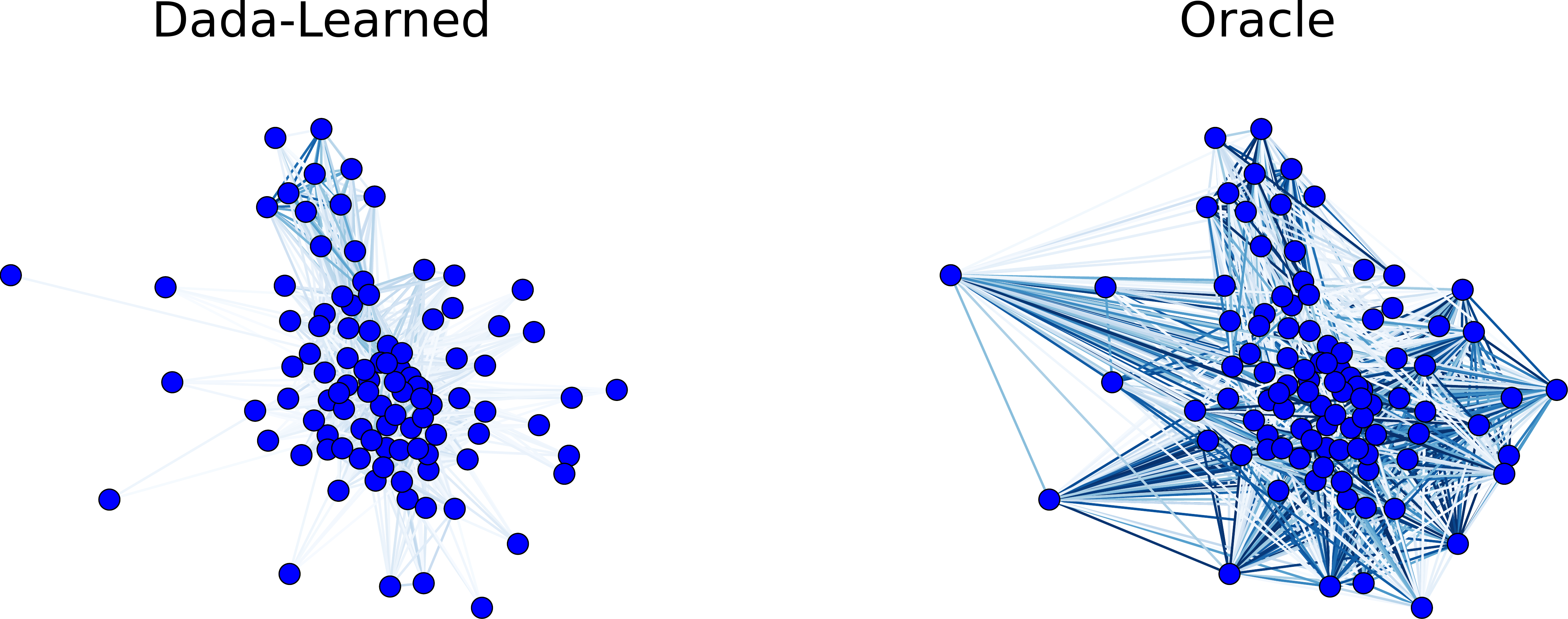}}
    \hspace{0.3cm}
    \subfigure[Impact of $\lambda$ on the sparsity of the learned graph and
    the test accuracy.]{
    \label{fig:sparsity}\includegraphics
    [width=.25\columnwidth]{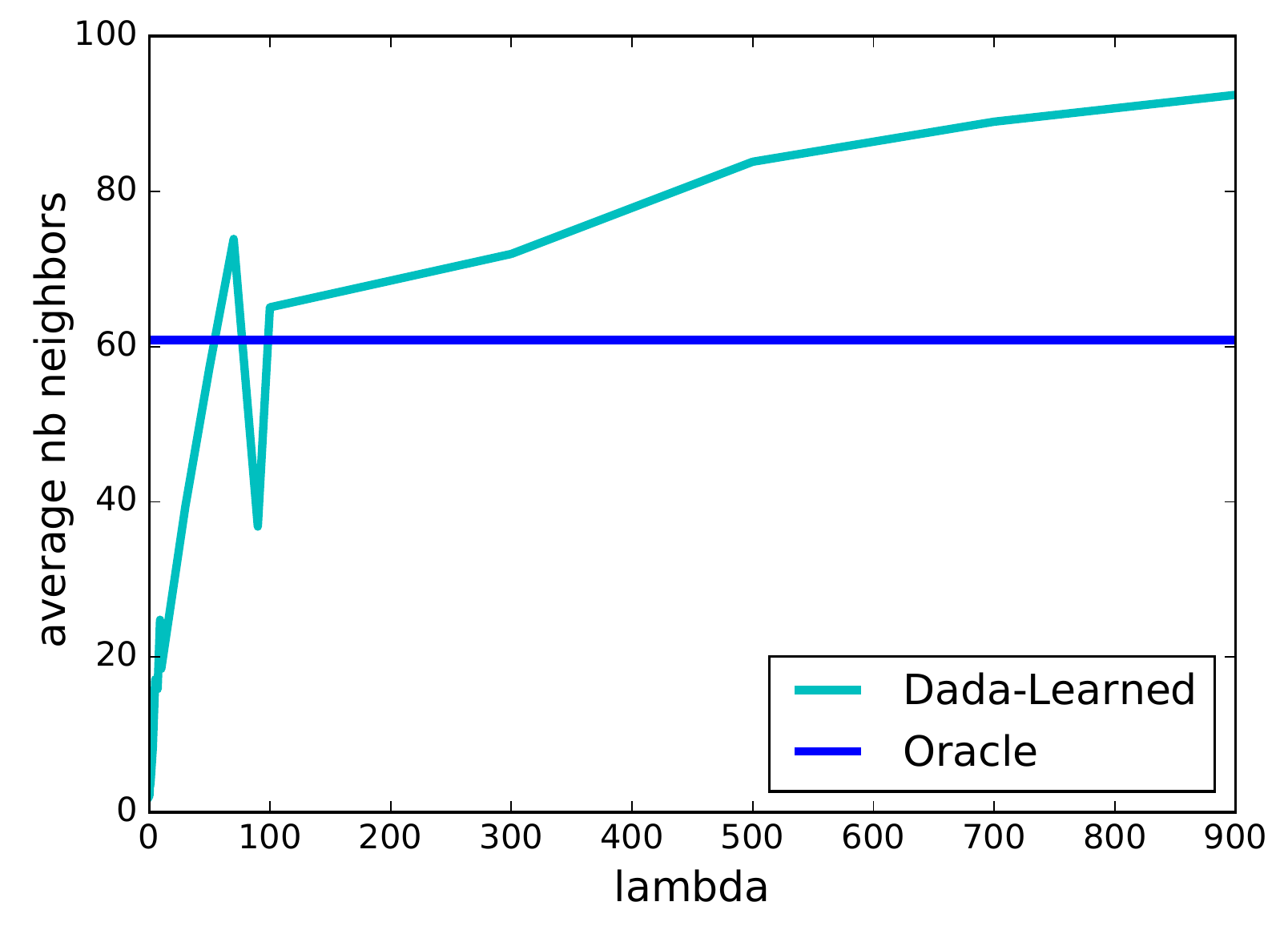}\includegraphics
    [width=.25\columnwidth]{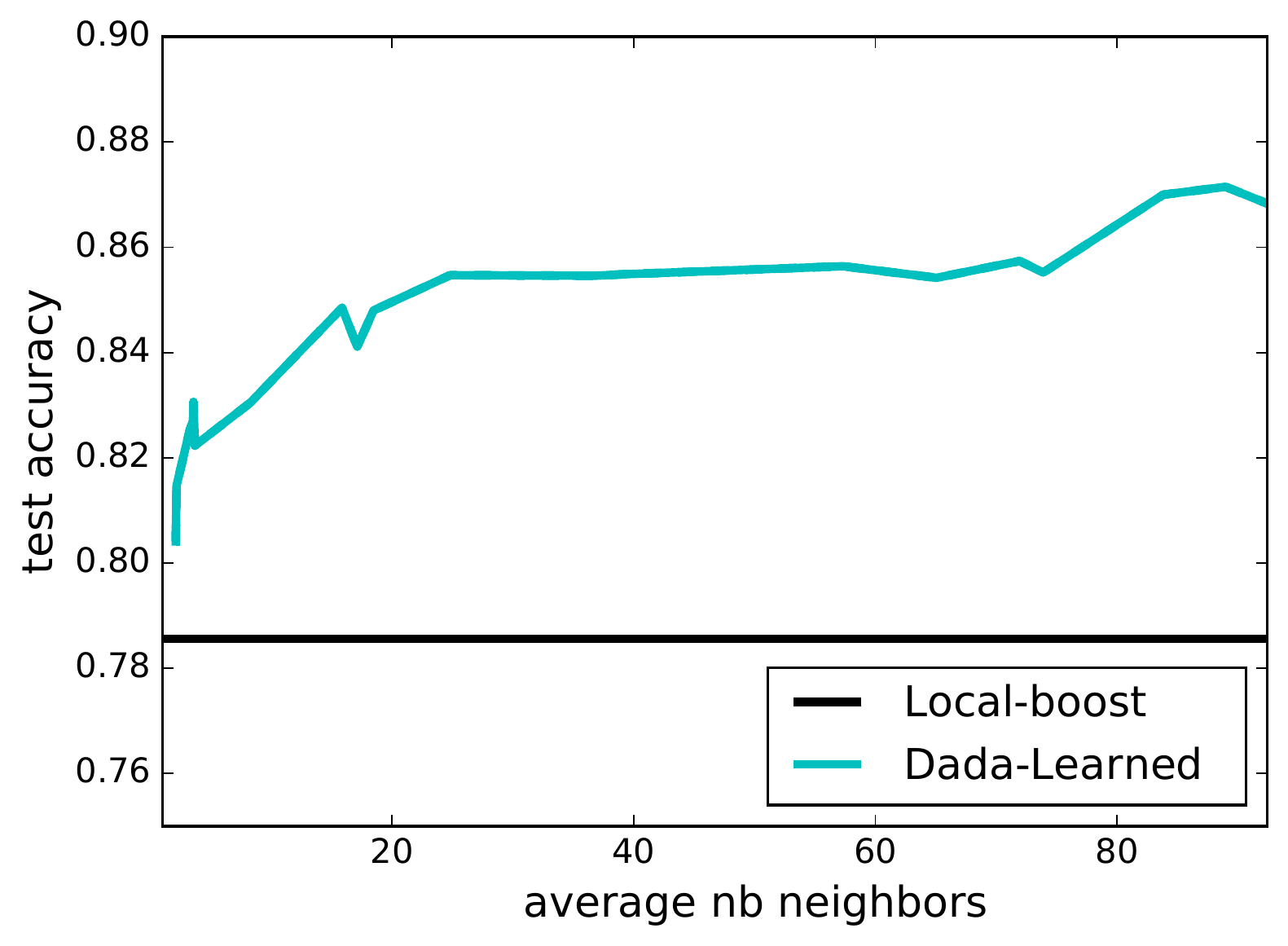}} \caption{Graph
    visualization and study of the impact of graph sparsity on {\sc
    Moons100}.}
    \label{fig:visu-sparsity}
\end{figure}

Figure~\ref{fig:100accuracies} shows the evolution of the training and test
accuracy over the iterations for the various approaches defined in the main
text.
The results are consistent with those presented for {\sc Moons100} in
the main text. They clearly show the gain in accuracy provided by our
method:\algonameF and \algonameL are successful in reducing the overfitting of
\localboost, and allow higher test accuracy than both \globalboost and
\persolin.
Again, we see that our strategy to learn the collaboration graph can
effectively make up for the absence of knowledge about the
ground-truth similarities between users.
At convergence, the learned graph has an average number of neighbors per node
$\shortexpect{|N_k|} = 42.64$, resulting in a communication complexity for updating the classifiers smaller than
the one of the ground-truth graph, which has $\shortexpect{|N_k|} = 60.86$ 
(see Figure~\ref{fig:visu-sparsity}).
We can make the graph even more sparse (hence reducing the
communication complexity of \algoname) by setting the hyper-parameter $\lambda$ to smaller values.
Of course, learning a sparser graph can also a negative impact on the accuracy
of the learned models.
In Figure~\ref{fig:sparsity}, we show this trade-off between the sparsity of
the graph and the test accuracy of the models for the {\sc Moons100} problem.
As expected, as $\lambda\rightarrow0$ the graph becomes sparser and
the test accuracy
tends to the performance of \localboost. 
Conversely, larger values of $\lambda$ induce denser graphs,
sometimes resulting in better accuracies but at the cost of higher
communication complexity.